\documentclass[arxiv]{melba}
\usepackage[T1]{fontenc}

\usepackage{subcaption} 
\usepackage{tabularx}
\usepackage{tikz} 
\usetikzlibrary{spy} 

\newtheorem{assumption}{Assumption}

\DeclareMathOperator*{\argmin}{arg\,min}

\newcommand{\R}{\mathbb{R}}

\newcommand{\N}{\mathbb{N}}

\newcommand{\dist}{\mathrm{dist}}

\newcommand{\Ac}{\mathcal{A}}
\newcommand{\Bc}{\mathcal{B}}
\newcommand{\Cc}{\mathcal{C}}
\newcommand{\Dc}{\mathcal{D}}
\newcommand{\Ec}{\mathcal{E}}

\newcommand{\Sc}{\mathcal{S}}
\newcommand{\Vc}{\mathcal{V}}
\newcommand{\Xc}{\mathcal{X}}

\newcommand{\vb}{\mathbf{v}}
\newcommand{\xb}{\mathbf{x}}
\newcommand{\yb}{\mathbf{y}}
\newcommand{\zb}{\mathbf{z}}

\newcommand{\eqdef}{\stackrel{\mathrm{def}}{=}}

\melbaid{2025:016}  
\doi{10.59275/j.melba.2025-f7b3}
\melbaauthors{Cazorla, Munier, Morin and Weiss}  
\email{clement.cazorla31@gmail.com}
\volume{3}
\firstpageno{367}  
\melbayear{2025}  
\datesubmitted{12/2024}  
\datepublished{08/2025}  

\ShortHeadings{MELBA Journal Sample Article}{Cazorla, Munier, Morin and Weiss}

\title{Sketchpose: Learning to Segment Cells with Partial Annotations}

\author{
    \firstname Clément \surname  Cazorla \aff{1,2}\orcid{0009-0003-4499-9004},
    \firstname Nathanaël \surname Munier \aff{1}\orcid{0000-0002-9264-2574},
    \firstname Renaud  \surname Morin \aff{2},
    \firstname Pierre  \surname Weiss \aff{1}\orcid{0000-0001-9083-214X}
}
\affiliations{
    \num 1 \addr Institut de Recherche en Informatique de Toulouse (IRIT), Institut de Mathématiques de Toulouse (IMT), Centre de Biologie Intégrative (CBI), Laboratoire de Biologie Moléculaire, Cellulaire et du Développement (MCD), Université de Toulouse, CNRS, Université Toulouse III – Paul Sabatier, Toulouse, France \\
    \num 2 \addr Imactiv-3D, Centre Pierre Potier, 1 place Pierre Potier, 31100 Toulouse, France \\
}

\abstract{
The most popular networks used for cell segmentation (e.g. Cellpose, Stardist, HoverNet,...) rely on a prediction of a distance map.
It yields unprecedented accuracy but hinges on \emph{fully annotated} datasets. 
This is a serious limitation to generate training sets and perform transfer learning. 
In this paper, we propose a method that still relies on the distance map and handles partially annotated objects.
We evaluate the performance of the proposed approach in the contexts of frugal learning, transfer learning and regular learning on regular databases.
Our experiments show that it can lead to substantial savings in time and resources without sacrificing segmentation quality.
The proposed algorithm is embedded in a user-friendly Napari plugin.}


\begin{document}

\twocolumn[\maketitle]



\section{Introduction}


 
 


Image segmentation plays a fundamental role in the analysis of biological images. It enables the extraction of quantitative information on diverse objects ranging from molecules, droplets, membranes, nuclei, cells, vessels or other structures. In modern biological research, accurate segmentation is often pivotal to better understand the mechanisms of life. 
The increasing availability of high-throughput imaging technologies has led to a surge in the quantity and complexity of image data, raising significant challenges and opportunities. Manual annotation of the resulting images is labor-intensive, time-consuming, and often impractical for large-scale datasets. Automated segmentation is therefore widely accepted as a critical step in biological research. 

\paragraph{A simplified history of cell segmentation}

Image segmentation has long been dominated by handcrafted algorithms. 
The processing pipelines typically combine popular tools such as linear filtering, thresholding \cite{otsu1979threshold}, morphological operations \cite{serra2012mathematical,legland2016morpholibj}, active contour models (Snake) \cite{kass1988snakes} or watershed \cite{vincent1991watersheds}. A significant issue with handcrafted approaches is that they are usually image-specific and rely on the manual tuning of a few complicated hyper-parameters. 
Although excellent performance can be achieved, it is often the work of a handful of talented people and these techniques are not broadly applicable.

The introduction of machine learning and especially random forests made image segmentation accessible to a much larger range of researchers. These techniques automatically combine and tune elementary image processing bricks. They are driven by a few easily interpretable user annotations. Embedded in well conceived software such as Ilastik \cite{berg2019ilastik} or Labkit \cite{arzt2022labkit}, these techniques heavily contributed to democratize image segmentation and classification.

Deep learning and convolutional neural networks played an important role in improving the segmentation performance around 2015. 
For instance, the popular U-Net architecture \cite{ronneberger2015u} increased the accuracy on some cell segmentation challenges by more than 10\%, which can be considered as a small revolution. 
This type of neural network architecture seems to be a good prior for segmenting ``natural'' images, as suggested by the so-called Deep Image Prior principle \cite{lempitsky2018deep}.
However, it can sometimes demonstrate limited effectiveness when it comes to separating nearby or touching objects. 
Many applications in biology involve densely packed objects (e.g. cells, nuclei) and a pixel-classification U-Net is often insufficient to perform a satisfactory analysis.
To address this issue, new architectures coming from computer vision such as Mask R-CNN~\cite{he2017mask} have been developed and continued improving the performance. 

Roughly at the same time, a few approaches (Deep watershed transform~\cite{bai2017deep}, Deep Regression of the Distance Map~\cite{naylor2018segmentation,Kumar2019}, StarDist~\cite{schmidt2018cell}, Hover-Net \cite{graham2019hover}, Cellpose~\cite{stringer2021cellpose}, Omnipose~\cite{cutler2022omnipose}) have been developed and generated results with an unprecedented quality. 
Despite certain differences, they all share a common underlying principle. 
The idea is to make a regression with respect to some \emph{distance function}. 
Given a set of annotated objects, a distance function to the objects centers or boundaries is computed. 
A convolutional neural network is then trained to predict the distance function rather than a binary map of the objects.
The gradient of this distance function points in opposite directions on each side of the boundary, which makes it possible to determine them with much greater precision. 
This principle created a new gap in the segmentation accuracy, especially for objects with touching boundaries. 

A current trend consists in involving the user in the training procedure.
This ``human in the loop'' principle was incorporated in CellPose 2.0 \cite{pachitariu2022cellpose}.
Users fully annotate patches of the segmented image, to adapt the neural network weights to the image at hand. 

It would be  hazardous to call these approaches the current ``state-of-the-art'', since this field is expanding extremely quickly.
However -- as of 2025 -- we can safely claim that algorithms based on the distance map are at the basis of some of the most popular and efficient cell segmentation methods.

\paragraph{Contributions}

This work stems from a practical observation: methods which rely on a regression to the distance function currently require exhaustive annotations. 
As the distance function is a global geometrical property, it is impossible to compute it using just a few sketches.
Hence, it is \textit{a priori} unclear how partial annotations can be used in this framework, see Figure~\ref{fig:main_question}.
Cellpose 2 gets around this problem by allowing the user to annotate patches of interest in their entirety.
Similarly, \cite{sugawara2023training} recently proposed a simple extension of Stardist and Cellpose by training the networks on a subset of completely annotated objects.
This is a time-consuming process that does not allow the expert to focus on local spots (e.g. a part of boundary) where the network clearly missed the segmentation.

\begin{figure}[h!]
	\centering
    \begin{subfigure}[b]{0.23\textwidth}
        \centering \includegraphics[height=0.8\textwidth]{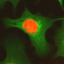}
        \caption{Image}
        \label{fig:init_img}
    \end{subfigure}
    ~
    \begin{subfigure}[b]{0.23\textwidth}
        \centering \includegraphics[height=0.8\textwidth]{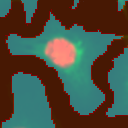}
        \caption{Full annotation}
        \label{fig:full}
    \end{subfigure}
    
    \begin{subfigure}[b]{0.23\textwidth}
        \centering \includegraphics[height=0.8\textwidth]{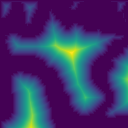}
        \caption{Distance map}
        \label{fig:distance}
    \end{subfigure}
	~
	\begin{subfigure}[b]{0.23\textwidth}
        \centering \includegraphics[height=0.8\textwidth]{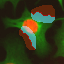}
        \caption{Partial annotation}
        \label{fig:partial}
    \end{subfigure}
	
	\caption{Partial VS full annotation. In (b), the complete annotation is used to compute the distance map shown in (c). In (d), it is unclear how to compute it from the partial object annotation.}
    \label{fig:main_question}
\end{figure}

In this paper, we introduce a novel idea that allows us to use the distance function even with partially annotated objects.
After drawing just a few regions and boundaries, the user can train a task-aware neural network. 
This approach capitalizes on the generalization capacity of neural networks, reducing the overall annotation effort without sacrificing accuracy. 
We explore the performance of the proposed architecture in 3 different settings:
\begin{itemize}
	\item \emph{Few-shot learning:} starting from random weights, we show that just a few \emph{partial} annotations are already enough to quickly realize complex cell segmentation analyses. This is interesting when faced with a problem for which no close pre-trained model exists.
	\item \emph{Transfer learning:} starting from Omnipose's optimized weights, we show that just a few clicks at locations where the segmentation is inaccurate lead to improved weights and fast adaptation to out-of-distribution images. This is the traditional field of transfer learning, domain adaptation, e.g.. Our contribution here is to show that this can be done with only a few scattered annotations.
	\item \emph{Large databases:} finally, we show that large, but partially annotated sets can also be used to train high performance neural networks. This is important since it can significantly accelerate the design of segmentation databases.  
\end{itemize}
This evaluation on both small and large-scale dataset, overall showcases the advantages of our approach in terms of time and resource savings. 
We developed a Napari plugin \cite{chiu2022napari} named Sketchpose, to assess its potential, ensure reproducibility of the results and provide an additional tool to the community. 
It relies on a modified version of the Omnipose \cite{cutler2022omnipose} algorithm. The plugin is currently being downloaded regularly, with 638 downloads to date.

\section{Methodology}

\subsection{Preliminary definitions and notations}

In all the paper $\mathcal{X}$ refers to the image domain, which can be understood as a discrete set of coordinates, or as a continuous domain depending on the context. 
In the discrete setting, we let $|\Xc|$ denote the number of pixels of $\Xc$. 

\begin{definition}
For an arbitrary set $\Sc \subset \Xc$, we let $\partial \Sc$ denote its boundary.    
We use the 4-connectivity (top, bottom, left, right) in the discrete setting.
\end{definition}

\begin{definition}[Point to set distance]
The distance from a point $\xb\in \Xc$ to a set $\Sc \subseteq \Xc$ is defined by    
\begin{equation}
    \dist(\xb,\Sc) \eqdef \inf_{\xb'\in \Sc} \|\xb-\xb'\|_2.
\end{equation}
\end{definition}

\subsection{Omnipose}

Our work is based on the Omnipose cell segmentation architecture \cite{cutler2022omnipose}.
In this section, we justify this choice, explain its founding principles and then demonstrate how they can be adapted to deal with partial annotations. 

\subsubsection{Why Omnipose}
Cellpose~\cite{stringer2021cellpose} has now become a standard in cell segmentation. 
Its excellent perfomance, processing speed, and ergonomic graphical interface make it a handy tool for every day cell biology image analysis.
However, it occasionally fails in scenarios involving complex and elongated objects. In such cases, it tends to produce over-segmentation, where neighboring objects are split in smaller fragments. 

The Omnipose algorithm~\cite{cutler2022omnipose} was conceived in order to address this limitation.
The main difference between Omnipose and Cellpose is the fact that the distance map is defined as the distance to the cell boundaries in Omnipose, while it is defined as a distance to a cell ``centroid'' in Cellpose. A weakness of the latter is that there is no canonical choice to define this center, hence Omnipose's choice seems more principled.
This explains our decision to choose and base our work on its architecture.


\subsubsection{The main principles}

\begin{figure*}
    \centering
    \includegraphics[width=0.9\linewidth]{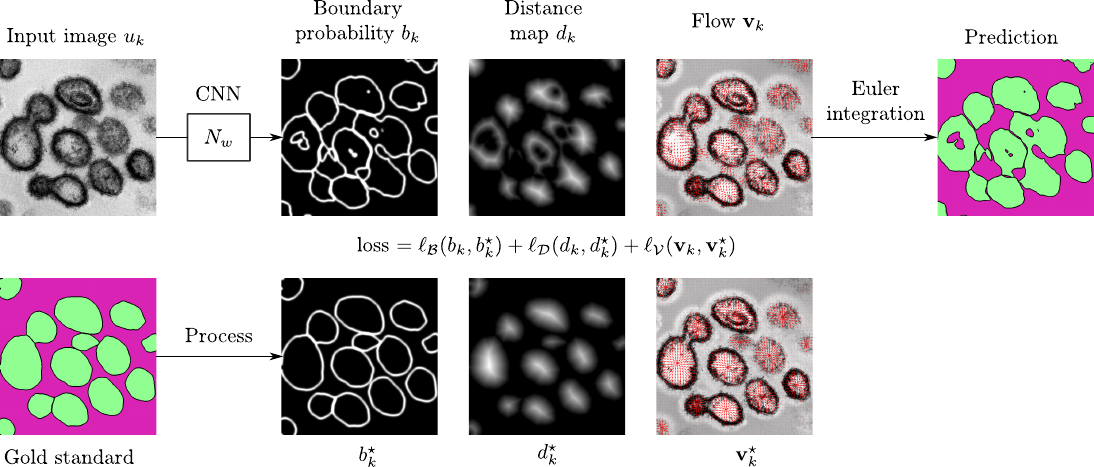}
    \caption{A sketch of the Omnipose training procedure \label{fig:omnipose_principle}}
\end{figure*}

Figure \ref{fig:omnipose_principle} summarizes the main ideas behind the Omnipose architecture and its training. 
Omnipose is based on a regular convolutional neural network (CNN), with a U-Net like architecture \cite{ronneberger2015u}. 
Given an input 2D image with $N$ pixels, the CNN can be seen as a mapping $N_w^{\textrm{omni}}$ of the form
\begin{equation}
    \begin{array}{ccccccccc}
    N_w^{\textrm{omni}} & : & \R^N & \to  & \R^{N}   & \times & \R^{N}   & \times & \R^{2N} \\
        &   & u & \mapsto & (N_w^b(u)& ,      & N_w^d(u) & ,      & N_w^{\vb}(u)) \\
    \end{array}.
\end{equation}
It depends on weights $w$ that should be optimized during a training stage.
It returns $3$ different outputs (illustrated on the top of Figure \ref{fig:omnipose_principle}):
\begin{itemize}
     \item $N_w^b(u) \equiv$ \emph{boundary probability}: at every pixel, the value of this image can be interpreted as a probability of being a boundary between the objects to segment.
     \item $N_w^d(u) \equiv$ \emph{distance map}: at a given pixel, the value of this map is equal to:
     \begin{itemize}
         \item The distance of the pixel to the closest object boundary, if the pixel is inside an object.
         \item $0$ (or a fixed negative value) elsewhere.
     \end{itemize} 
     \item $N_w^{\vb}(u) \equiv$ \emph{flow field}: can be interpreted as the gradient of the distance map. It is an essential feature of the Cellpose and Omnipose architectures. Ultimately, the flow is used through a procedure called Euler integration to generate a segmentation mask. This is illustrated on the top right of Figure \ref{fig:omnipose_principle}. 
 \end{itemize}  

\subsubsection{The original loss definition\label{sec:losses}}

The original training stage involves a collection of $K\in \N$ images $(u_k)_{1\leq k \leq K}$ together with their \emph{exhaustive} segmentation masks.
For every image $u_k$ in the dataset, an algorithm creates the gold standard boundary probability $b_k^\star$, distance map $d_k^\star$ and flow field $\vb_k^\star$. 
This is illustrated on the bottom of Figure \ref{fig:omnipose_principle}.

The weights $w$ of the neural network are then optimized so as to minimize a loss function that compares the output of the CNN with the gold standard:
\begin{equation}
    \inf_{w} \mathrm{loss}^{\mathrm{omni}}(w) \eqdef \frac{1}{K}\sum_{k=1}^{K} \ell_{\mathcal{B}}(b_{k}, b_k^\star) + \ell_{\mathcal{D}}(d_{k}, d_k^\star) + \ell_{\mathcal{V}}(\vb_{k}, \vb_k^\star),
\end{equation}
where $b_k = N_w^b(u_k)$, $d_k = N_w^d(u_k)$, $\vb_k = N_w^{\vb}(u_k)$.  

In the original Omnipose implementation available on \href{https://github.com/kevinjohncutler/omnipose}{GitHub}, the different losses were defined as follows:
\begin{itemize}[leftmargin=*]
    \item \textbf{Boundary loss} $\ell_{\mathcal{B}}$:  

        This term compares the predictions $b$ to $b^\star$ using the following loss:
\begin{equation}
\ell_{\mathcal{B}}(b, b^\star) \eqdef \frac{\lambda_{\mathcal{B}}}{|\Xc|}\sum_{\xb\in \Xc} g(b[\xb], b^\star[\xb]), 
\end{equation}
where $g:\R\times \R\to \R$ combines a sigmoid and a binary cross entropy loss.
    
    \item \textbf{Distance loss} $\ell_{\mathcal{D}}$: 
    
This loss calculates a weighted mean squared error between the predicted distance fields and the ground truth distance fields. It is defined as
    \begin{equation*}
        \ell_{\mathcal{D}}(d, d^\star) \eqdef \frac{\lambda_{\mathcal{D}}}{|\Xc|} \sum_{\xb\in \mathcal{X}} (d[\xb] - d^\star[\xb])^2 \cdot \rho[\xb],
    \end{equation*}
    where $\rho\in \R^N$ is a weight image with higher values around the gold standard boundaries.
    
    \item \textbf{Flow loss} $\ell_{\mathcal{V}}$: 

This loss is defined as a weighted sum of three losses $\ell_{\mathcal{V}}= \ell_{\mathcal{V}}^1 + \ell_{\mathcal{V}}^2 + \ell_{\mathcal{V}}^3$. 
The first one is a mean squared error loss:
\begin{equation}
    \ell_{\mathcal{V}}^1 (\vb, \vb^\star) \eqdef \frac{\lambda_{\mathcal{V},1}}{|\Xc|} \sum_{\xb\in \mathcal{X}} \left\|\vb[\xb] - \vb^\star[\xb]\right\|_2^2 \cdot \rho[\xb].
\end{equation}

The second one compares the norms of the vector fields:
\begin{equation}
    \ell_{\mathcal{V}}^2(\vb, \vb^\star) \eqdef \frac{\lambda_{\mathcal{V},2}}{|\Xc|} \sum_{\xb\in \mathcal{X}} (\|\vb[\xb]\|_2 - \|\vb^\star[\xb]\|)^2 \cdot \rho[\xb].
\end{equation}

The third one aims to minimize the distance between trajectories generated through the ground truth and predicted flows. 
Trajectories starting from an initial point $\zb$ can be generated by simple explicit Euler discretization:
\begin{align*}
    \xb_{0}(\zb)   &\eqdef \zb \\
    \xb_{l+1}(\zb) &\eqdef \xb_l(\zb) + \Delta t \cdot\vb[\xb_{l}(\zb)]
\end{align*}
\begin{align*}
\xb^\star_{0}(\zb)   &\eqdef \zb \\
\xb^\star_{l+1}(\zb) &\eqdef \xb_l^\star(\zb) + \Delta t  \cdot \vb^\star[\xb^\star_{l}(\zb)].
\end{align*}
where $\Delta t$ is a step-size. 
Letting  $L\in \N$ denote an integration time, the ``Euler'' loss then becomes:
\begin{equation}
    \ell_{\mathcal{V}}^3 (\vb, \vb^\star) \eqdef \frac{\lambda_{\mathcal{V},3}}{|\Xc|} \sum_{\zb\in \Xc} \sum_{l=1}^L \|\xb_l(\zb) - \xb_l^\star(\zb)\|_2^2.
\end{equation}
It measures how two trajectories generated by Euler integration using the ground truth and predicted vector fields deviate.
This loss is implemented in the torchVF library by \cite{TorchVF}. 
For more information, we refer the reader to the \href{https://github.com/ryanirl/torchvf/blob/main/article/first_draft.pdf}{related report}.
\end{itemize}

An inspection of the code reveals that the different weights have been set empirically as: $\lambda_{\mathcal{B}}=10$, $\lambda_{\mathcal{D}}=2$, $\lambda_{\mathcal{V},1}=2$, $\lambda_{\mathcal{V},2}=2$, $\lambda_{\mathcal{V},3}=1$.
\begin{remark}
    The different losses have probably been combined by trial and error to produce the best possible results. 
    However, there are clear redundancies in the definitions of the losses, for instance $\ell_{\mathcal{V}}^1$, $\ell_{\mathcal{V}}^2$ and $\ell_{\mathcal{V}}^3$ are all measuring the distance between flows using different metrics.
    In our implementation, we tried to simplify the losses as much as possible, while still maintaining a good performance.
\end{remark}

\subsection{Adapting to partial annotations}

All the principles described above heavily depend on an exhaustive segmentation of the cells. 
Indeed, the distance functions and gradient flows -- which are instrumental to define the loss functions -- are global properties which do change heavily if the objects boundaries are incomplete. 
In this section, we describe the main methodological contribution of this paper, which will allow us to handle partial boundaries.

\subsubsection{The gold standard}

\begin{figure}[h!]
    \centering
    \begin{minipage}[c]{0.4\columnwidth} 
        \centering
        \begin{subfigure}[b]{\textwidth}
            \centering \includegraphics[height=3.3cm]{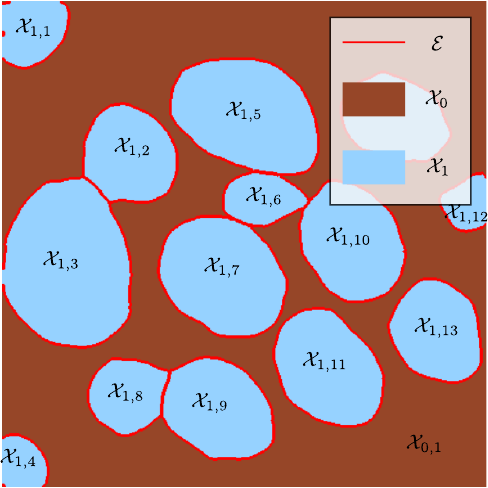}
            \caption{Gold-standard}\label{fig:annot1}
        \end{subfigure}
        \begin{subfigure}[b]{\textwidth}
            \centering \includegraphics[height=3.3cm]{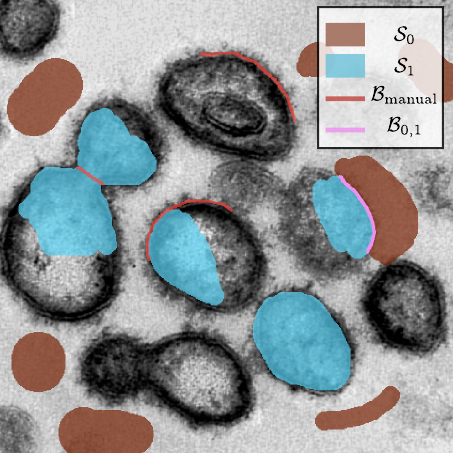}
            \caption{Valid annotation}\label{fig:annot2}
        \end{subfigure}
    \end{minipage}
    \begin{minipage}[c]{0.4\columnwidth} 
        \centering
        \begin{subfigure}[b]{\textwidth}
            \centering \includegraphics[height=2cm]{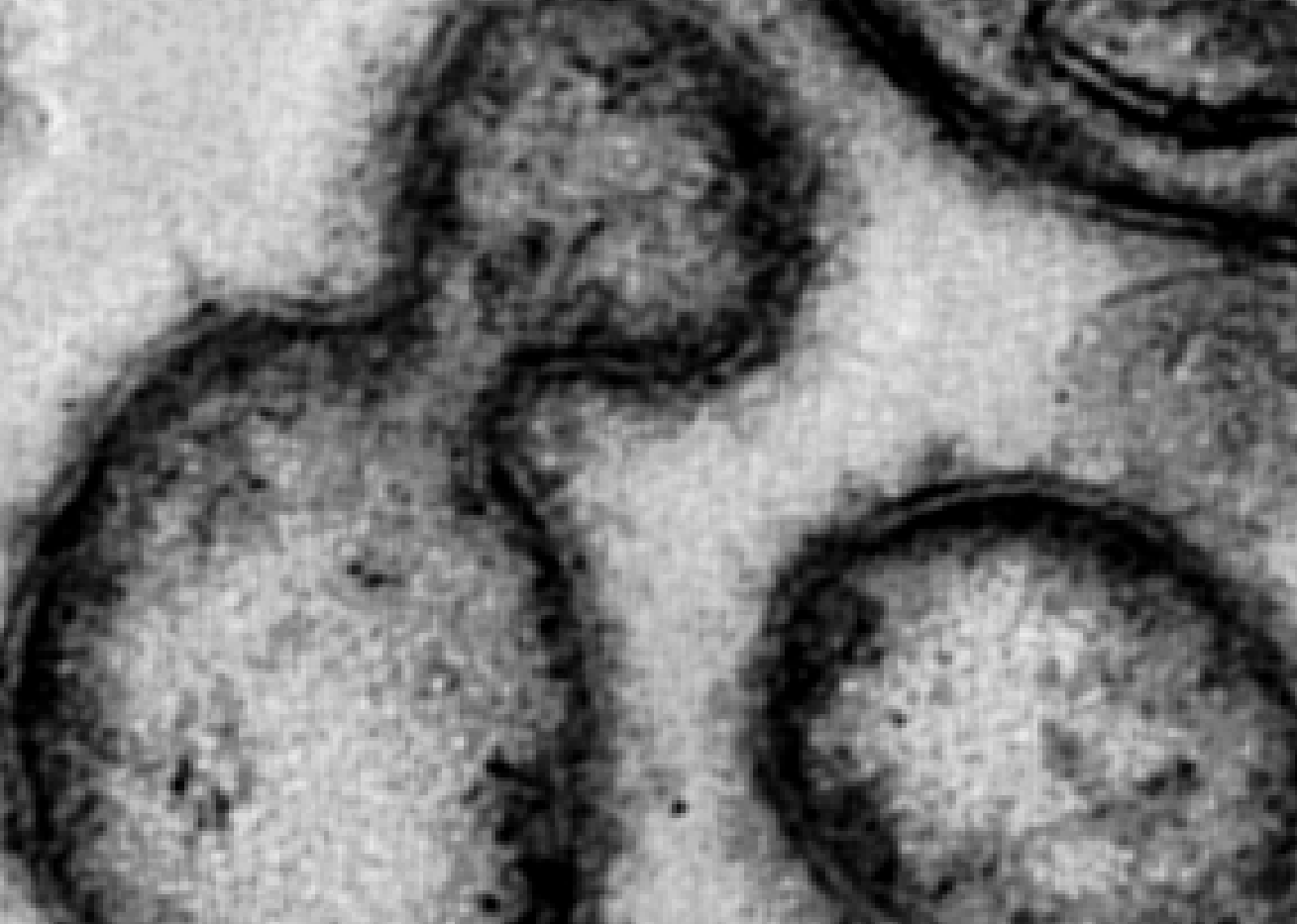}
            \caption{Image}\label{fig:original}
        \end{subfigure}
        \begin{subfigure}[b]{\textwidth}
            \centering \includegraphics[height=2cm]{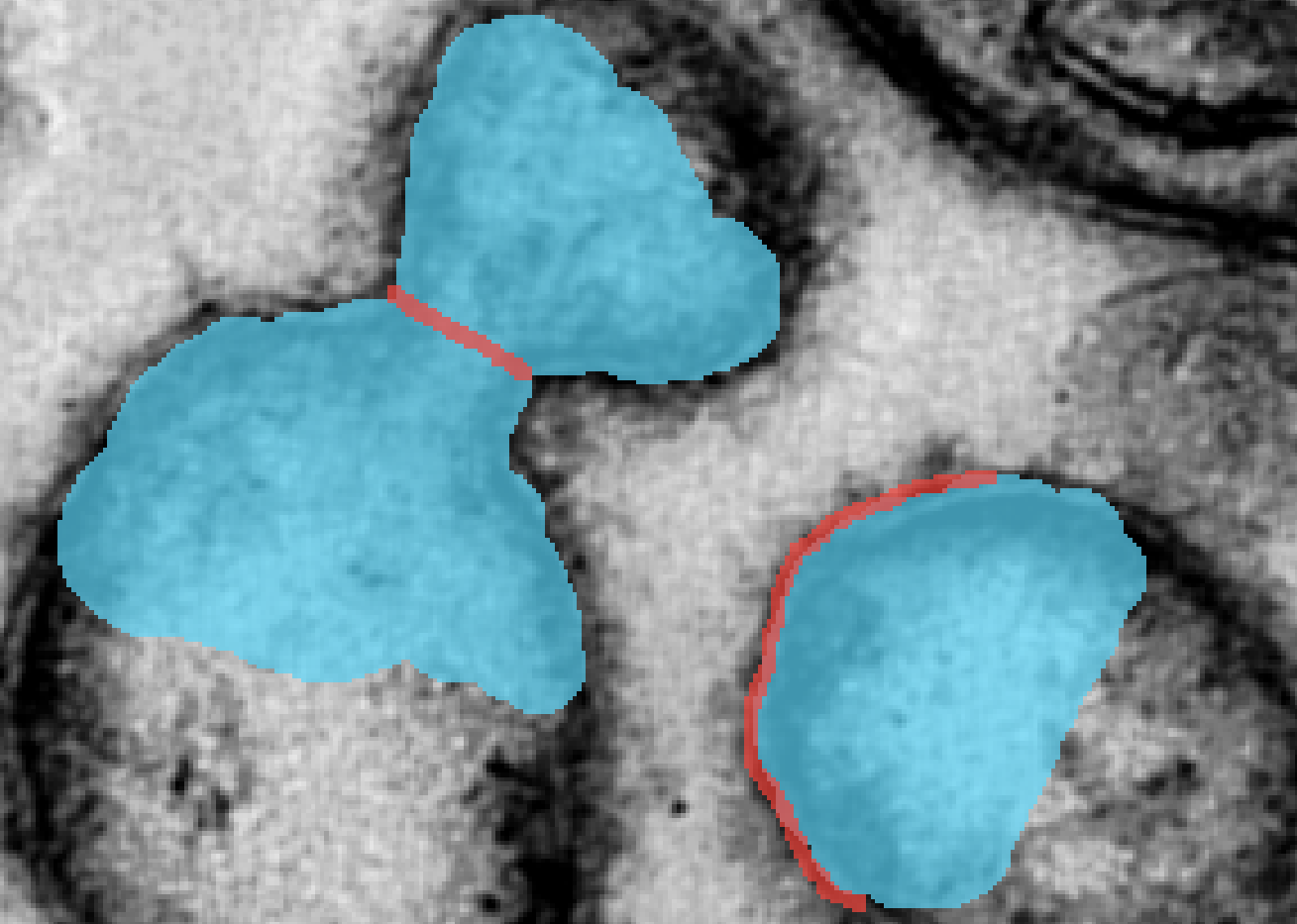}
            \caption{Admissible}\label{fig:admissible}
        \end{subfigure}
        \begin{subfigure}[b]{\textwidth}
            \centering \includegraphics[height=2cm]{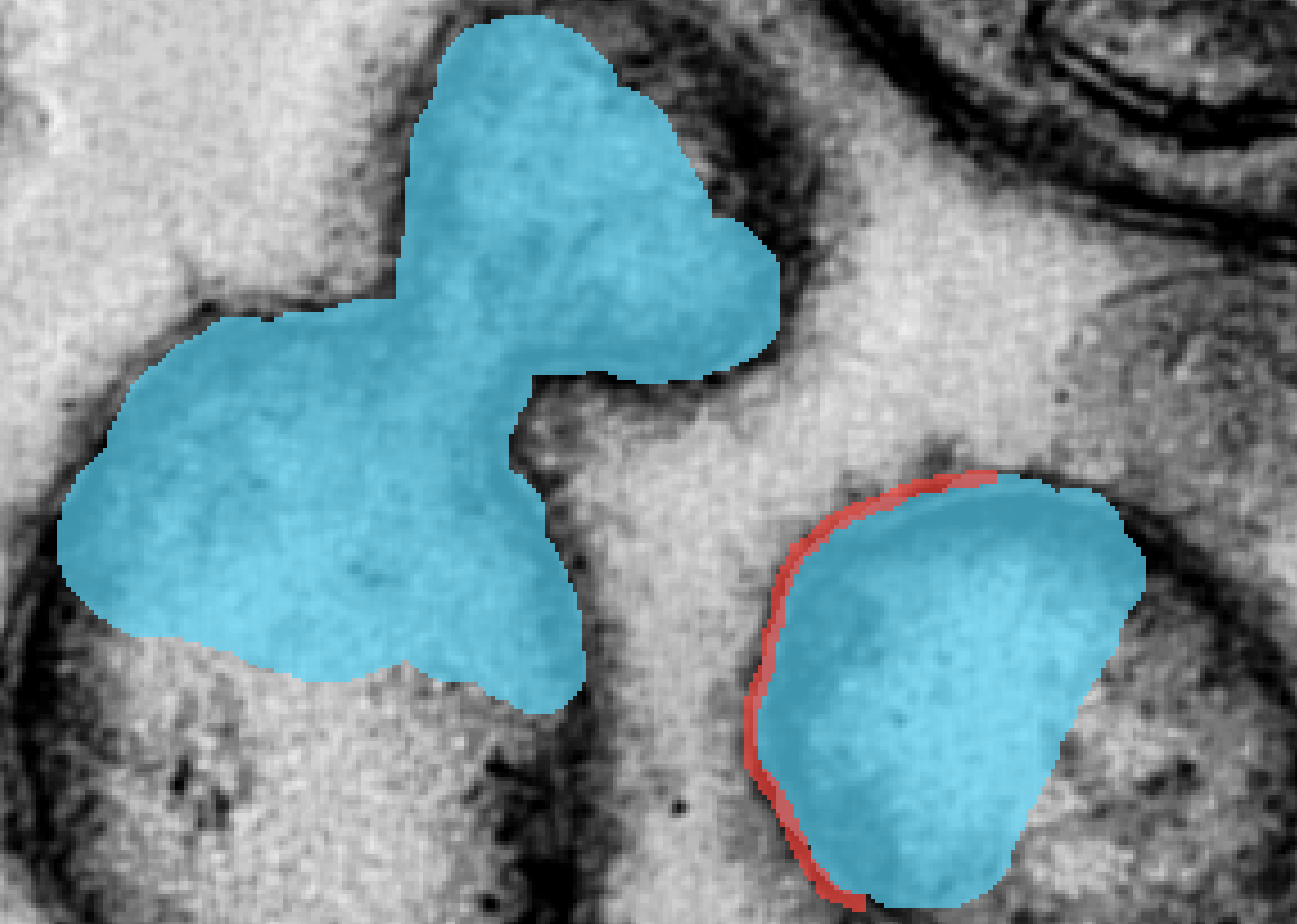}
            \caption{Not admissible}\label{fig:notadmissible}
        \end{subfigure}
    \end{minipage}
    \caption{(Left) Ground-truth and an admissible annotation set. (Right) If a stroke contains multiple objects, the object boundaries have to be drawn. In this example, two nuclei are present under the blue bow-tie-shaped region. Therefore, a manual boundary has to be added in the center.}\label{fig:combined}
\end{figure}


\begin{table}[ht]
\centering
\caption{
Summary of notations.}
\label{tab:notations}
\renewcommand{\arraystretch}{1.1}
\begin{tabularx}{\columnwidth}{@{}lX@{}}
\toprule
\textbf{Notation} & \textbf{Description} \\
\midrule
$\mathcal{X} = \mathcal{X}_0 \cup \mathcal{X}_1$ & Image domain \\
$\mathcal{X}_0$ & True background \\
$\mathcal{X}_1$ & True foreground \\
$\mathcal{S}_0$ & Background strokes \\
$\mathcal{S}_1$ & Foreground strokes \\
$\mathcal{E}$ & True boundaries \\
$\mathcal{B} = \mathcal{B}_{\mathrm{manual}} \cup \mathcal{B}_{0,1}$ & User-defined boundaries \\
$\mathcal{D}$ & Valid distance set \\
\bottomrule
\end{tabularx}
\end{table}

The notations are summarized in Table \ref{tab:notations}.
We assume that the domain $\Xc = \Xc_0 \sqcup \Xc_1$ is partitioned with the background set $\Xc_0$ and the foreground set $\Xc_1$.
A difficulty in instance segmentation is that multiple objects may exist within the connected components of a region $\Xc_i$. 
To differentiate them, we let $(\Xc_{i,j})_{1\leq j \leq J_i}$ denote a partition of the set $\Xc_i$ as different objects within a similar class. 
For instance in Figure \ref{fig:annot1}, the foreground set $\Xc_1$ is split in 13 components. 
A connected component of $\Xc_1$ can be split as $\Xc_{1,2}\cup \Xc_{1,3}$.
The background set $\Xc_0$ is split in a single component $\Xc_{0,1}$.

We let 
\begin{equation}
    \Ec \eqdef \bigcup_{i\in \{0,1\}} \bigcup_{j=1}^{J_i} \partial \Xc_{i,j}
\end{equation}
denote the set of all \emph{edges} (or object boundaries) within the image. 
It is depicted in red in Figure \ref{fig:annot1}.

\subsubsection{The annotation set}

The input of our neural network is a set of ``sketches'' or strokes drawn by the user. 
We let $\Sc_0$ and $\Sc_1$ denote the strokes describing the background and foreground respectively.
They are depicted in brown and blue respectively in Figure \ref{fig:annot2}.
The intersection of the brown and blue strokes define natural boundaries. 
We can indeed construct the \emph{touching boundaries} $\Bc_{0,1}$ between different strokes as 
$$
\Bc_{0,1} \eqdef \overline{\Sc_{0}} \cap \overline{\Sc_{1}},
$$
where $\overline{\Xc}$ is the closure of $\Xc$ in the continuous setting and the interface between neighboring pixels in the discrete setting. 

In addition, the user can delineate other boundaries, denoted $\Bc_{\textrm{manual}}$, to separate touching objects within a class. 
We can concatenate all the boundaries to obtain a complete boundary set $\Bc$ defined as
\begin{equation}
  \Bc = \Bc_{\textrm{manual}} \bigcup \Bc_{0,1}.
\end{equation}

For the algorithm to work properly, we require the following set of assumptions.
\begin{assumption}[Assumptions on the strokes\label{ass:ass1}]
\ 
\begin{itemize}
    \item The strokes correctly separate the background and foreground: $\Sc_0\subseteq \Xc_0$ and $\Sc_1\subseteq \Xc_1$.
    \item The strokes do not overlap: $\Sc_0 \cap \Sc_1 = \emptyset$. This is actually forced by our Napari interface.
    \item The boundaries $\Bc$ are a subset of the exact boundaries $\Ec$, that is:
            \begin{equation} \label{eq:eq_ass1}
            \Bc \subseteq \Ec.    
            \end{equation}
    \item If the stroke $\Sc_i$ contains multiple objects, then the boundaries between the objects \emph{need to be completely drawn} with  $\Bc_{\textrm{manual}}$ (see Figure \ref{fig:admissible}). Letting $\Sc\eqdef \Sc_0 \cup \Sc_1$ denote the complete stroke set drawn by the user, this condition reads:
        \begin{equation}
            \Ec \cap \Sc \subseteq \Bc. \label{eq:boundaries_drawn}
        \end{equation}
\end{itemize}
\end{assumption}


\subsubsection{The main observation}\label{section:valid_dist}

The main result we will use to define and certify our algorithm is summarized in the following theorem.

\begin{theorem}[The valid distance set\label{thm:valid_dist_set}]
Let $\Cc\Bc \eqdef \partial \Sc_0 \cup \partial \Sc_1 \cup \Bc$ denote the complete set of annotation boundaries and define the \emph{valid distance set} $\Dc$ as
\begin{equation}
    \Dc \eqdef \left\{x\in \Sc, \dist(x,\Bc) \leq \dist\left(x, \Cc\Bc\right)\right\}.
\end{equation}
The following relationships hold:
\begin{align*}
    \textrm{For all } \quad x\in \Dc, \quad &\dist(x, \Ec) = \dist(x,\Bc). \\
    \textrm{For all } \quad x\in \Sc_0\cup \Sc_1, \quad &\dist(x, \Cc\Bc) \leq \dist(x,\Ec).
\end{align*}
\end{theorem}
The proof of this theorem is given in Appendix~\ref{appendixA}.
This theorem should be understood as follows. The first identity informs us that we can compute the \emph{exact distance map} $\dist(x,\Ec)$ to the set of exact boundaries $\Ec$ on the valid set $\Dc$. This set can be computed using only the partial annotations of the boundaries $\Bc\subseteq \Ec$ and the different semantic regions $\Sc_i\subseteq \Xc_i$.
The second inequality tells us that we have an information everywhere on the strokes $\Sc_0$ and $\Sc_1$.
Moreover, in the case of total annotations, we get $\Dc = \Xc$ and the proposed idea will lead to a training equivalent to the one in Omnipose and  we can see the proposed setting as a generalization. 
Figure \ref{figure:D} schematically summarizes Theorem \ref{thm:valid_dist_set}.

\begin{figure}[!ht]
 \centering
 \includegraphics[width=\columnwidth]{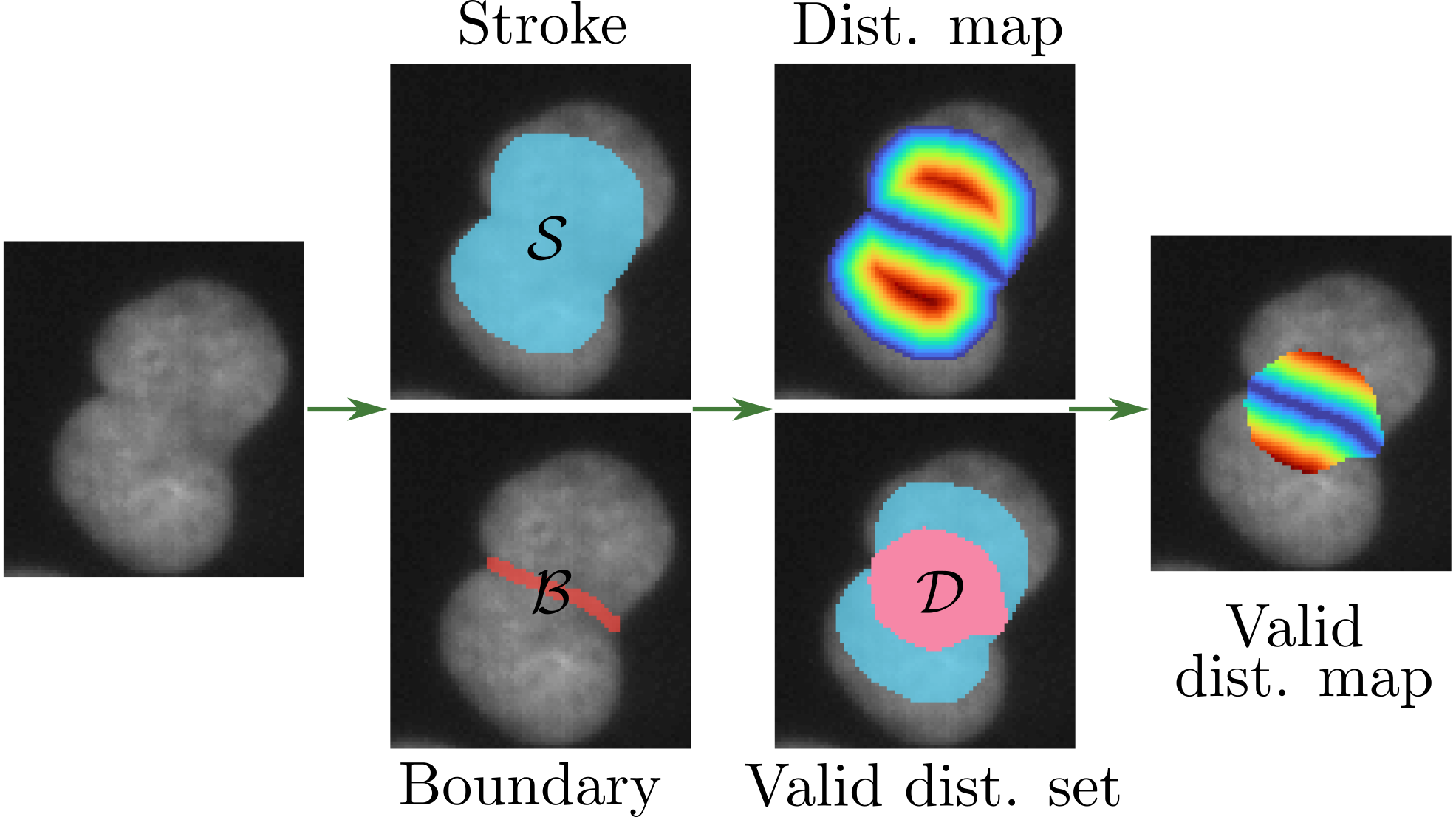}
  \caption{Illustration of the valid distance set theorem. All the pixels in $\Dc$ are closer to $\Bc$ than to the boundaries of $\Sc$. 
  The colormap used to represent the distance map exhibits a progressive shift in colors, transitioning from blue to red.\label{figure:D}}
\end{figure}

\subsubsection{Adapting the training}

\paragraph{A simpler architecture}

The prediction of the boundaries $N_w^b(u)$ is not necessary, since only the distance map $N_w^d(u)$ and the flow field $N_w^{\vb}(u)$ are needed to compute the final masks.
Hence, we keep the same U-Net architecture, but remove the channel associated to the boundaries:
\begin{equation}
    \begin{array}{ccccccc}
    N_w^{\textrm{sketch}} & : & \R^N & \to  & \R^{N}   & \times & \R^{2N} \\
        &   & u & \mapsto & (N_w^d(u) & ,      & N_w^{\vb}(u) )  \\
    \end{array}.
\end{equation}

\paragraph{Different summation sets}

Equipped with the valid distance set $\Dc$, we are ready to adapt the losses to cope with partial annotation. In Omnipose, the losses $\ell_{\Bc}$, $\ell_{\Dc}$, $\ell_{\Vc}^1$ and $\ell_{\Vc}^2$ are defined by summation over the set $\Xc$ (see paragraph \ref{sec:losses}). With partial annotation, the gold standard is not properly defined on this set and we therefore need to change the summation sets.

Based on Theorem \ref{thm:valid_dist_set}, the losses related to the distance set and to the flows become:
\begin{align*}
    \ell_{\mathcal{D}}^{\textrm{partial}}(d, d^\star) &\eqdef \frac{\lambda_{\mathcal{D}}}{|\Dc|} \sum_{\xb\in \mathcal{\Dc}} (d[\xb] - d^\star[\xb])^2, \\
    \ell_{\mathcal{V}}^{{\textrm{partial}}} (\vb, \vb^\star) &\eqdef \frac{\lambda_{\mathcal{V},1}}{|\Dc|} \sum_{\xb\in \Dc} \left\|\vb[\xb] - \vb^\star[\xb]\right\|_2^2. 
\end{align*}
We replaced $\sum_{\xb\in \Xc}$ by $\sum_{\xb\in \Dc}$ in section~\ref{sec:losses} and discarded the weights $\rho$. This means that we compare the ground truth and prediction only where it makes sense to do so.

\paragraph{Dealing with inequalities}

Until now, we just used the first identity in Theorem \ref{thm:valid_dist_set}, but the second inequality brings some additional information.
We propose to integrate it in the training through the additional asymmetric loss:
\begin{equation}
    \ell_{\Dc}^{\textrm{ineq}} (d, d^\star) \eqdef \frac{\lambda_{{\Dc}, \textrm{ineq}}}{|\Sc_1\cup \Sc_2|} \sum_{\zb\in \Sc_1\cup \Sc_2} \textrm{ReLU}^2\left(d[\zb] - d^\star[\zb]\right).
\end{equation}

\paragraph{Putting it all together}

We can now define the total sketchpose loss as:
\begin{equation}
 \mathrm{loss}^{\textrm{sketch}}(w) = \ell_{\mathcal{V}}^{{\textrm{partial}}} (\vb, \vb^\star) +  \ell_{\mathcal{D}}^{\textrm{partial}}(d, d^\star) +  \ell_{\Dc}^{\textrm{ineq}} (d, d^\star) 
\end{equation}
where $d, \vb$ are obtained by the simplified neural network $N_w^{\textrm{sketch}}$.

\subsection{The Sketchpose plugin}

A significant part of this work lies in the development of a user-friendly graphical interface to train and use the neural network. 
It is integrated in Napari \cite{chiu2022napari}, which is well suited to embedding the Python/PyTorch codes at the core of our approach.

Sketchpose can be easily installed through either the pip package manager or Napari's \cite{chiu2022napari} built-in interface. 
A \href{https://sketchpose-doc.readthedocs.io/en/latest/index.html}{detailed documentation} can be accessed by clicking on this hyperlink. 
It offers step-by-step instructions illustrated by short videos, to assist users in effectively testing all the capabilities of the plugin. 

The user directly draws a few strokes for the background, foreground and boundaries. The brush size can be adjusted similarly to usual paint software. An entire stroke boundary can be added to the boundary set by a double right-click. The training and drawing can be achieved in parallel to target the places where the segmentation is inaccurate.

The networks can be initialized by random weights or existing pre-trained weights.
The multi-threaded plugin's architecture makes it possible to annotate, train and observe the current segmentation results simultaneously. 
The user can annotate regions where the segmentation is inaccurate in priority, hence reducing the annotation time. 
The predictions can be restricted to a bounding box at each epoch of the training process, to reduce the processing time,
which is particularly helpful for large scale images.
Finally, users can work with a single image or a set of images for the inference and training steps.

\section{Experiments}

In this paragraph, we conduct several experiments to explore three distinct use cases of the method. 
\begin{itemize}
    \item Learning from a limited set of annotations on a single image with randomly initialized neural network weights. 
    \item Learning from a limited set of annotations on a single image, starting from a pre-trained neural network.
    \item Learning with randomly initialized weights using a large dataset with sparse annotations. We study the impact of the percentage of labeled pixels (10\%, 25\%, 50\% and 100\%) on the segmentation quality, when training on thousands of cells.
\end{itemize}

After describing the metrics used for validation, we will turn to the practical results. 
For all the experiments, we used a single Nvidia RTX5000 GPU with 16 Gb.

\subsection{Evaluation metrics}

To quantify the predictions quality, we enumerate the true positives (TP), the true negatives (TN) and the false positives (FP). A true positive is an object in the gold standard that can be matched to an object in the prediction with an Intersection over Union (IoU) criterion higher than a threshold $\tau$. We let $TP(\tau)$ denote the total number of true positives. 
The total number of estimated objects without matches is denoted $FP(\tau)$ (for false positives). 
The total number of gold standard objects without valid matches is denoted $FN(\tau)$ (for false negatives). 
Utilizing these values, we compute the object detection accuracy metric ($OA(\tau)$) \cite{caicedo} for each image using the formula:
\begin{equation*}
    OA(\tau)  = \frac{TP(\tau)}{TP(\tau) + FP(\tau) + FN(\tau)}.
\end{equation*}

The reported object dectection accuracy is then computed as the average over all images in the test set. 

Additionally, we computed the average DICE and the Aggregated Jaccard Index defined as follows:
\begin{align*}
J_{\text{aggregated}}(A,B) &= \frac{1}{N} \sum_{i=1}^{N} \frac{|A_i \cap B_i|}{|A_i \cup B_i|}, \\
\text{average DICE}(A,B) &= \frac{1}{N} \sum_{i=1}^{N} \frac{2 \cdot |A_i \cap B_i|}{|A_i| + |B_i|},
\end{align*}
where $A$ and $B$ are a dataset and its groundtruth.

\subsection{Training from scratch on a single image}

In this section, we will showcase several results achieved while training from scratch on a small set of images.
The tests are made on a variety of biological structures (dendritic cells, osteoclasts, bacteria, insect eggs, adipose tissue, artistic image of cells).

\subsubsection{Training details}

For this experiment, the model have been trained for 100 epochs ($\approx$ 2 minutes) with a batch size of 16 and image flips for data augmentation. 

\subsubsection{Staphylococcus aureus}

In the example of Figure \ref{fig:training_cells_colors}, we use a microscopy image of methicillin-resistant \textit{Staphylococcus aureus} (MRSA) infections, from European Commission, \textit{Horizon Magazine} (2020), \href{https://projects.research-and-innovation.ec.europa.eu/en/horizon-magazine/can-we-reverse-antibiotic-resistance}{``Can we reverse antibiotic resistance?''}. It is reused under the European Commission’s reuse policy.

\begin{figure*}[ht!]
  \centering
  \begin{subfigure}[t]{0.22\linewidth}
    \centering
    \includegraphics[width=\linewidth]{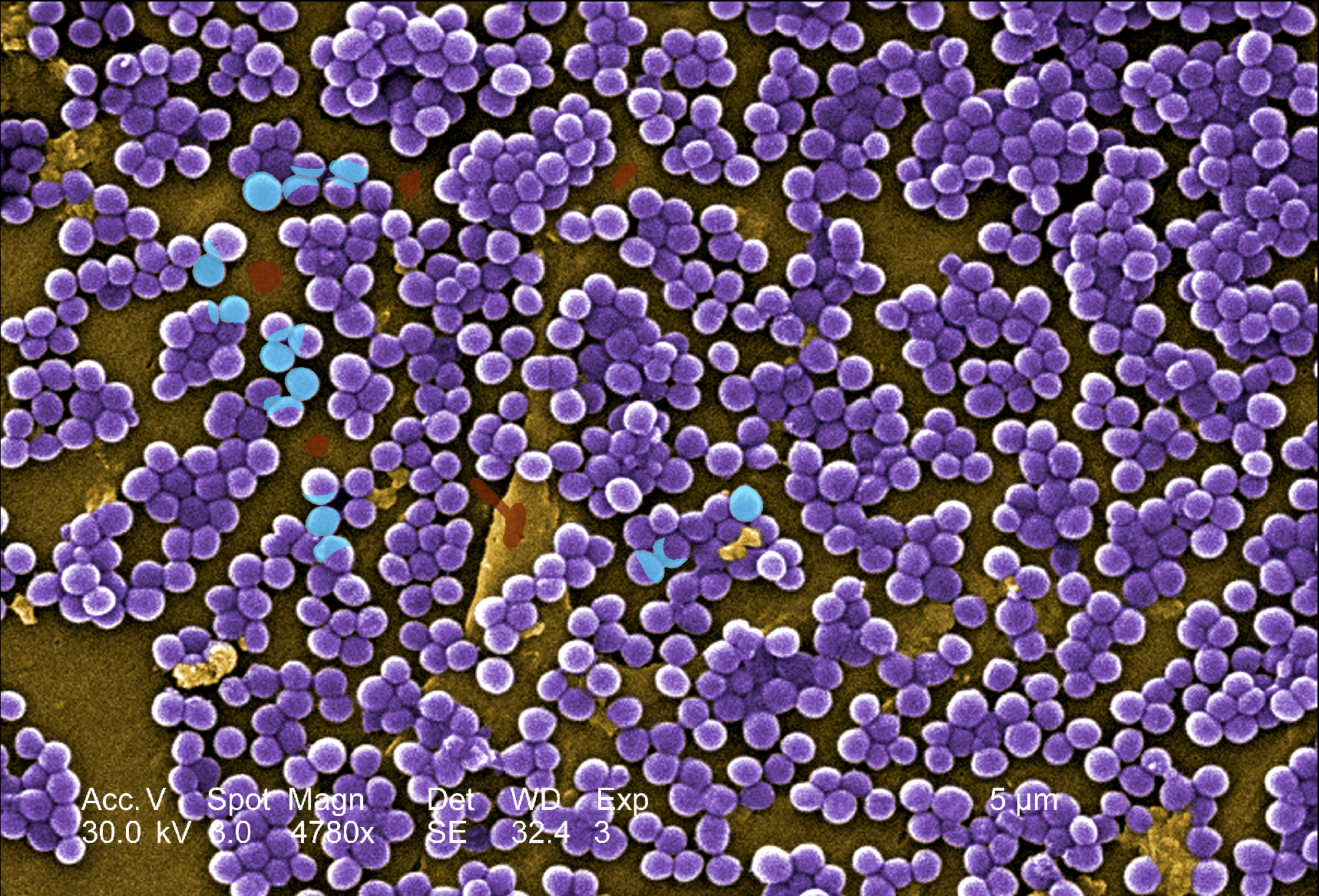}
    \caption{Sparse labels}
    \label{fig:image2}
  \end{subfigure}
  \hfill
  \begin{subfigure}[t]{0.22\linewidth}
    \centering
        \begin{tikzpicture}[spy using outlines={rectangle, color=red, magnification=3, size=1.25cm, connect spies}]
      \node[anchor=south west,inner sep=0] (image) at (0,0) {\includegraphics[width=\linewidth]{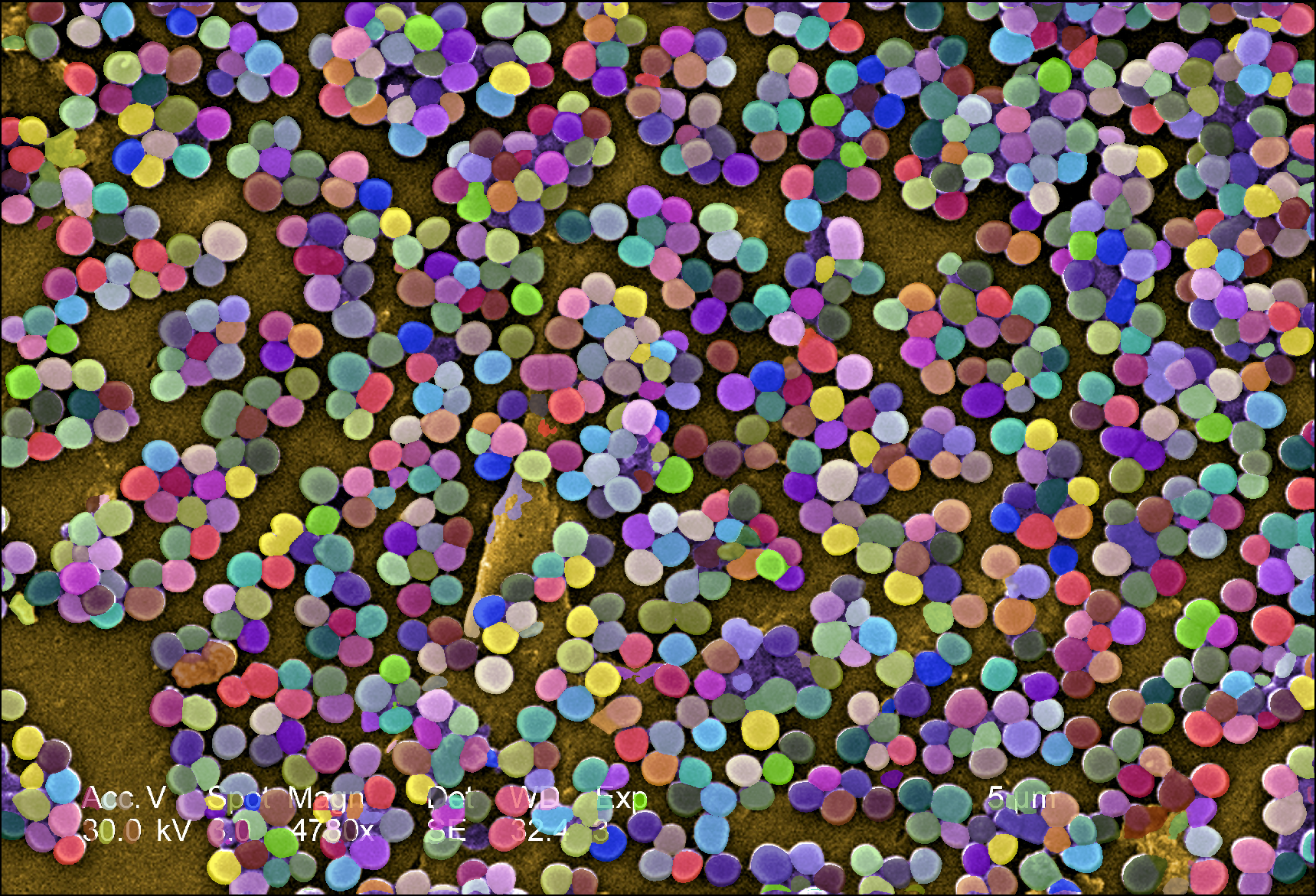}};
      \spy on (3.8,0.8) in node at (0.6,1.25);
    \end{tikzpicture}
    \caption{Omnipose cyto2 model}
    \label{fig:image3}
  \end{subfigure}
  \hfill
  \begin{subfigure}[t]{0.22\linewidth}
    \centering
        \begin{tikzpicture}[spy using outlines={rectangle, color=green, magnification=3, size=1.25cm, connect spies}]
      \node[anchor=south west,inner sep=0] (image) at (0,0) {\includegraphics[width=\linewidth]{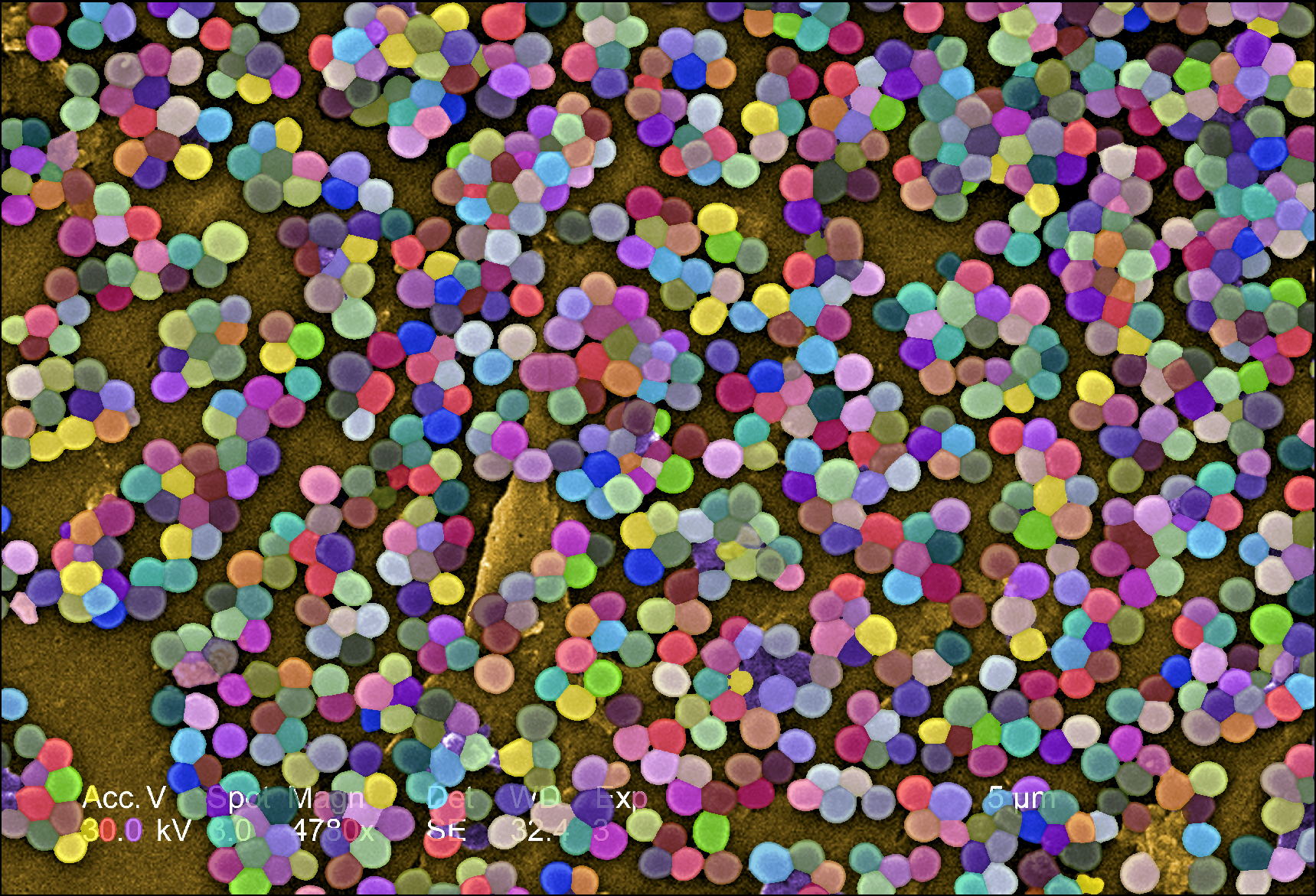}};
      \spy on (3.8,0.8) in node at (0.6,1.25);
    \end{tikzpicture}
    \caption{Sketchpose result, trained from scratch ($\approx$ 2')}
    \label{fig:image3}
  \end{subfigure}
  \hfill
  \begin{subfigure}[t]{0.22\linewidth}
    \centering
    \includegraphics[width=0.8\linewidth]{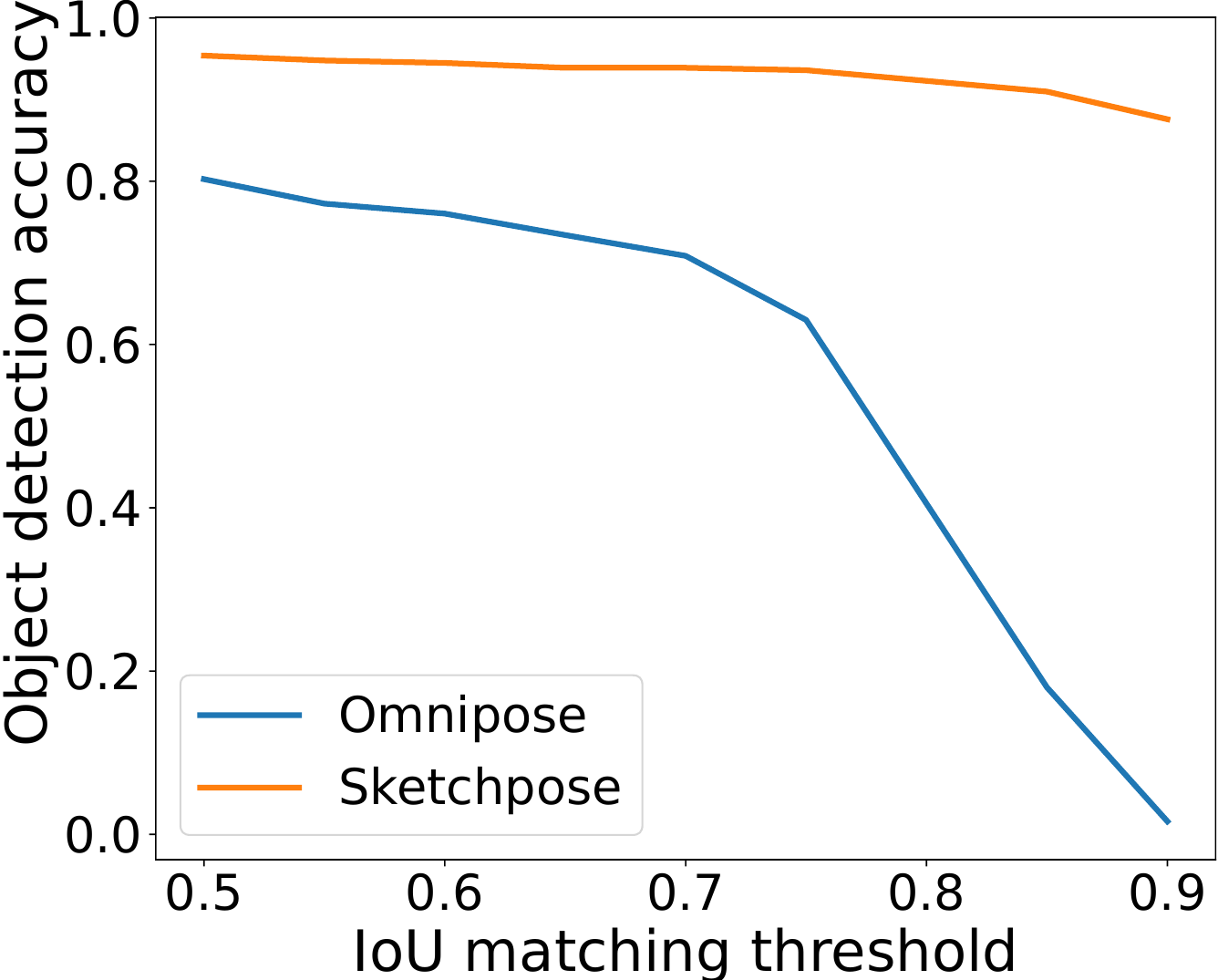}
    \caption{Evaluation of the segmentation quality}
    \label{fig:cells_color_acuracy}
  \end{subfigure}
  \caption{(a,b,c) A training from scratch with sparse labels (in blue on the left image). Image credit: Janice Carr, Jeff Hageman, USCDCP. (d) Evaluation of segmentation quality. Omnipose results are: DICE=0.92, Jaccard index=0.82. Sketchpose: DICE=0.99, Jaccard index=0.96.}
  \label{fig:training_cells_colors}
\end{figure*}

After drawing for less than one minute and training for 100 epochs ($\approx$ 2'), we achieve a much better result than the trained model of Omnipose (see Figure \ref{fig:training_cells_colors}). The quality metrics is shown in Figure \ref{fig:cells_color_acuracy}.

\subsubsection{Eggs on a tree leaf} \label{sect:eggs}

In this section, we picked an image from the Omnipose dataset, which likely represents eggs of an insect on a tree leaf. 
At first sight, the segmentation task is uneasy, since the objects are tightly connected, with identical textures and blurry boundaries.
We first annotated a subset of 5 eggs in Figure \ref{fig:egg_annot1} with a minimal amount of background. 
The segmentation result after training is already surprisingly good in Figure \ref{fig:egg_res1}, but some objects are not detected, and others are merged. 
We annotated 2 extra eggs in Figure \ref{fig:egg_annot2}. With this extra information, retraining the network now produces a near perfect segmentation mask, with a single error (2 pink eggs on the left). 
This experiment illustrates a unique feature of Sketchpose: it is possible to interactively annotate while training. 
This offers a possibility to label a minimum amount of regions to reach the desired output. 
This principle sometimes called ``active learning'' or ``human-in-the-loop'' \cite{budd2021survey} is significantly enhanced by using partial annotations and the user-friendly Napari interface.

\begin{figure*}[ht!]
  \centering
  \begin{subfigure}[t]{0.195\linewidth}
    \centering
    \includegraphics[width=\linewidth]{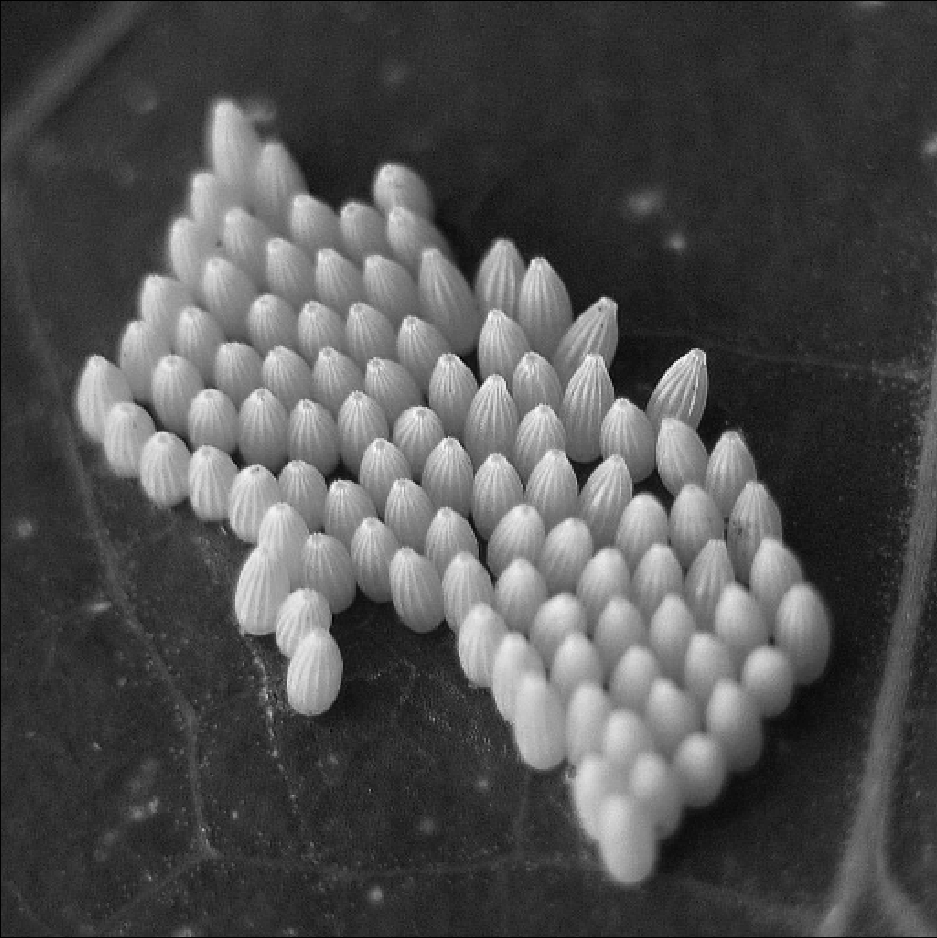}
    \caption{Image}
    \label{fig:egg_raw}
  \end{subfigure}
  \hfill
  \begin{subfigure}[t]{0.195\linewidth}
    \centering
    \includegraphics[width=\linewidth]{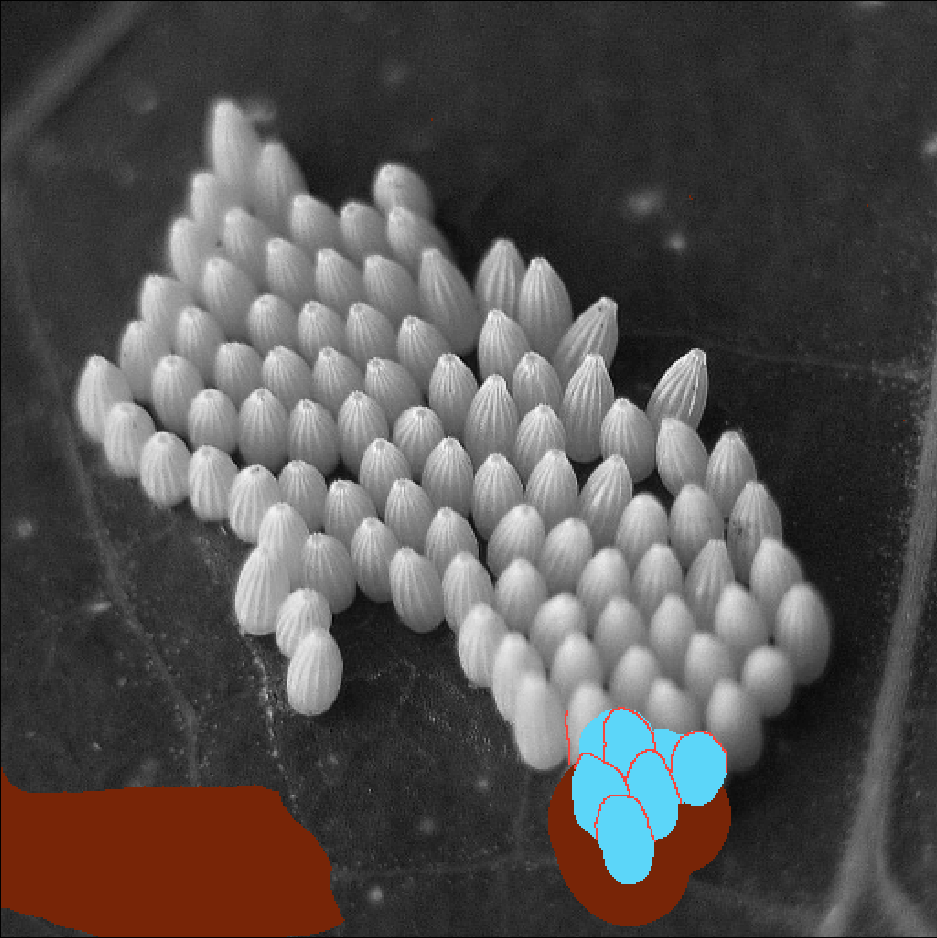}
    \caption{Labeled Set 1 (LS1)}
    \label{fig:egg_annot1}
  \end{subfigure}
  \hfill
  \begin{subfigure}[t]{0.195\linewidth}
    \centering
    \includegraphics[width=\linewidth]{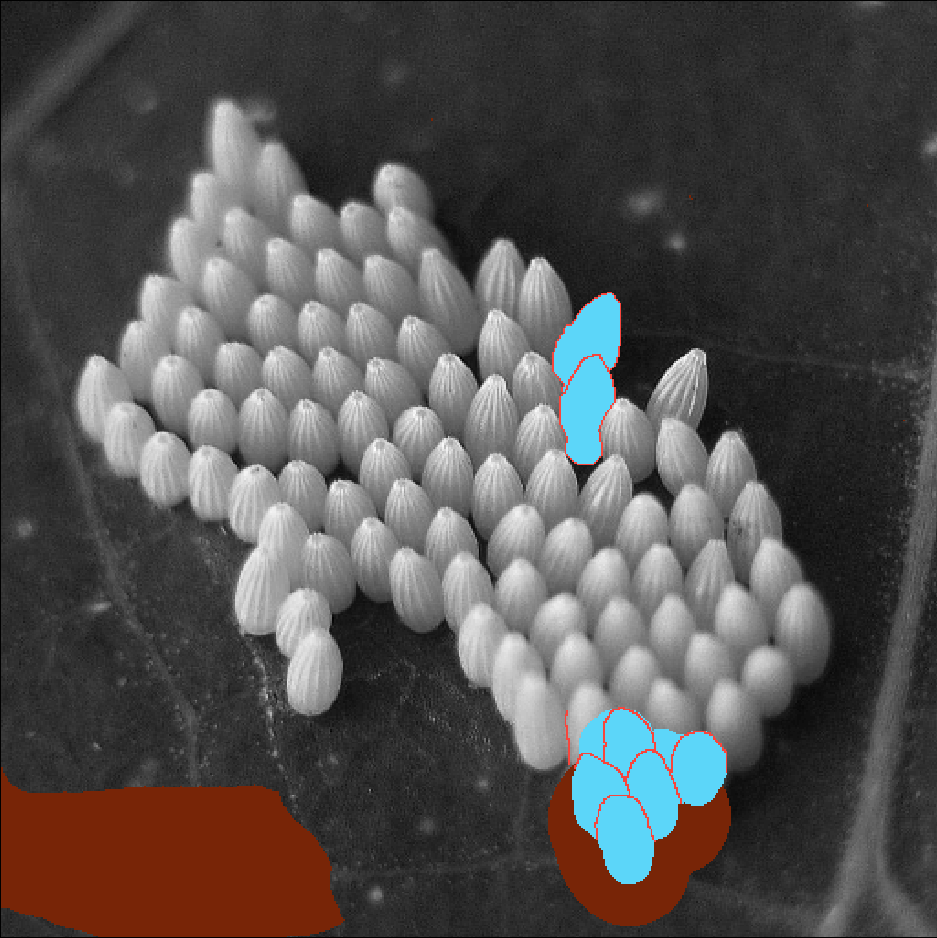}
    \caption{Labeled Set 2 (LS2)}
    \label{fig:egg_annot2}
  \end{subfigure}
  \hfill
  \begin{subfigure}[t]{0.195\linewidth}
    \centering
    \includegraphics[width=\linewidth]{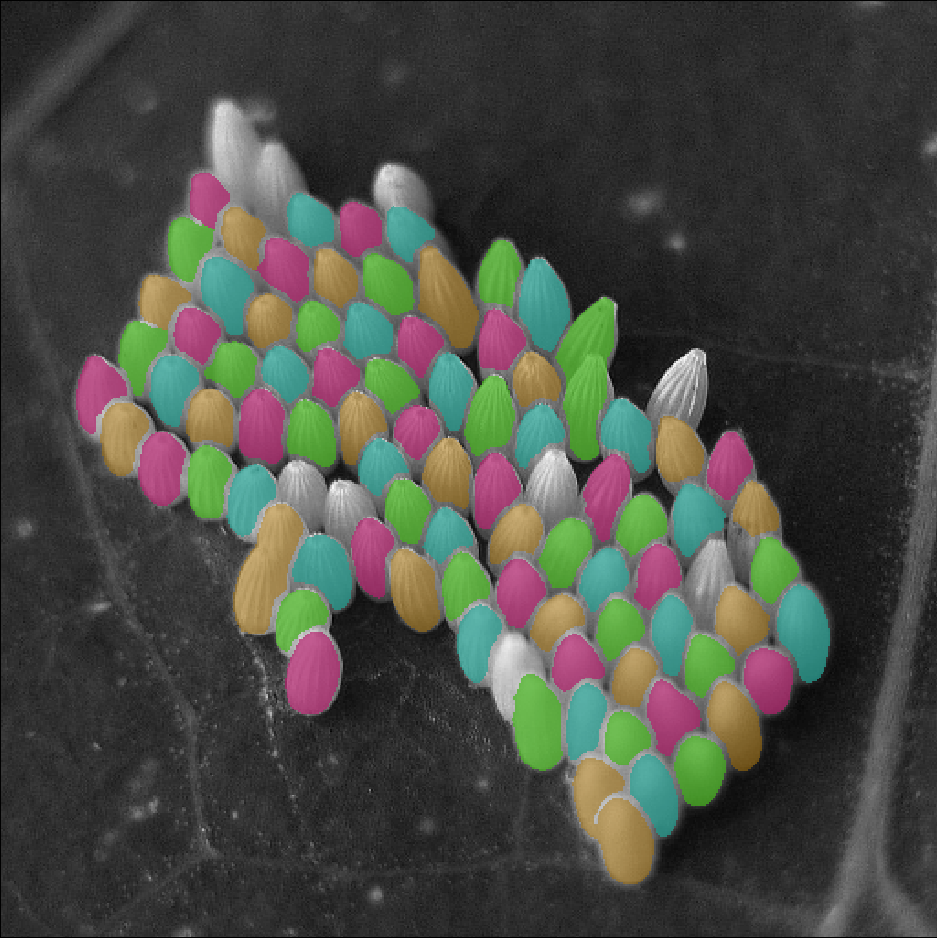}
    \caption{Sketchpose LS1 $\approx$ 10'}
    \label{fig:egg_res1}
  \end{subfigure}
  \hfill
  \begin{subfigure}[t]{0.195\linewidth}
    \centering
    \includegraphics[width=\linewidth]{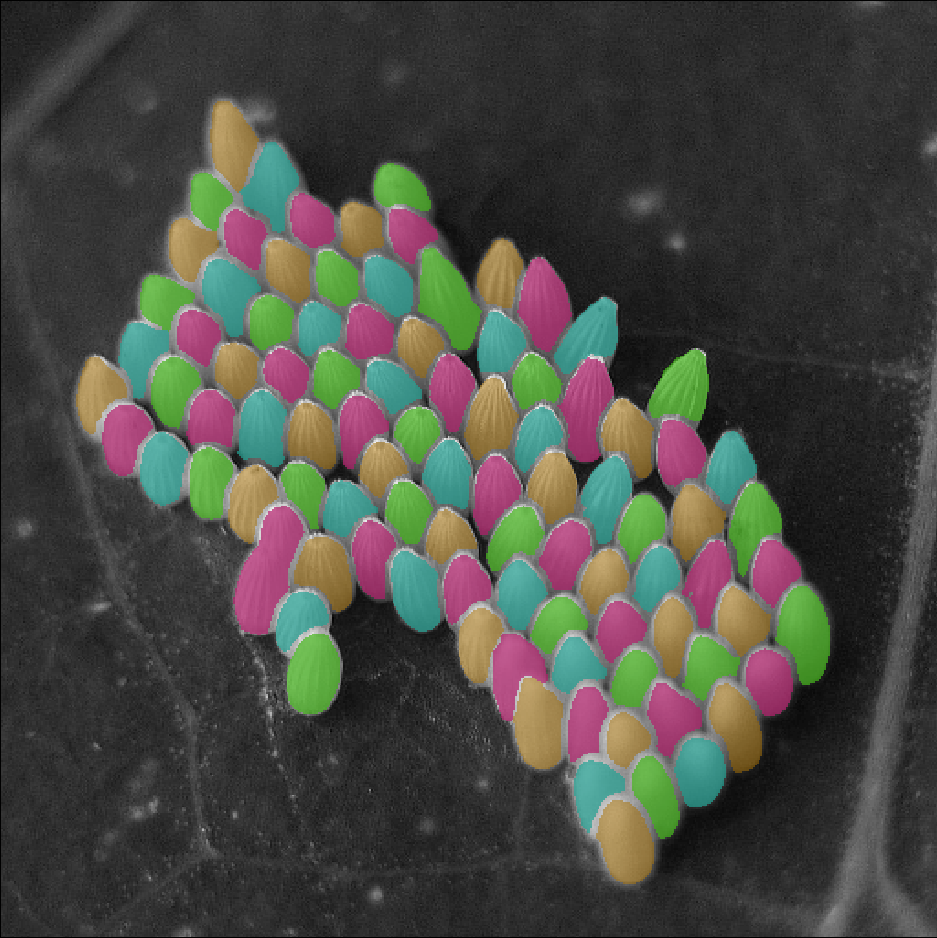}
    \caption{Sketchpose LS2 $\approx$ 3'}
    \label{fig:egg_res2}
  \end{subfigure}
  \caption{Progressive training in Sketchpose. In this example, we show that it is possible to improve the segmentation performance of Sketchpose by progressively annotating at places where the network failed. Here, a quite minimal annotation set is enough to near perfectly separate the eggs on the leaf.}
  \label{fig:training_eggs}
\end{figure*}

\subsection{Transfer learning on a single image}
In this section, we explore the feasibility of improving pre-trained weights using transfer learning.

\subsubsection{Training details}
As for the previous experiment, the model have been trained for 100 epochs ($\approx$ 2 minutes) with a batch size of 16 and image flips for data augmentation.

\subsubsection{Bacteria segmentation}

Bacteria are often used as biological models (e.g. in DNA studies). A precise segmentation can be difficult to achieve, because they have elongated shapes and can be clustered. 

The Omnipose~\cite{cutler2022omnipose} model was initially conceived to address the shortcomings of Cellpose for this task. 

Figure \ref{fig:training_bact} shows how transfer learning with sparse annotations can improve the Omnipose results by separating touching bacterias. Figure \ref{fig:bact_acuracy} shows a quantitative comparison of both methods. As can be seen, Sketchpose's adapted weights provide much higher performance. 
A visual inspection indicates that all objects have been correctly separated, apart from the cluster touching the boundary on the bottom left.

\begin{figure*}[ht!]
  \centering
  \begin{subfigure}[t]{0.225\linewidth}
    \centering
    \begin{tikzpicture}[spy using outlines={rectangle, color=red, magnification=3, size=1.3cm, connect spies}]
  \node[anchor=south west, inner sep=0] (image) at (0,0) {
    \includegraphics[clip, trim={0pt 50 0 0}, height=\linewidth, angle=90]{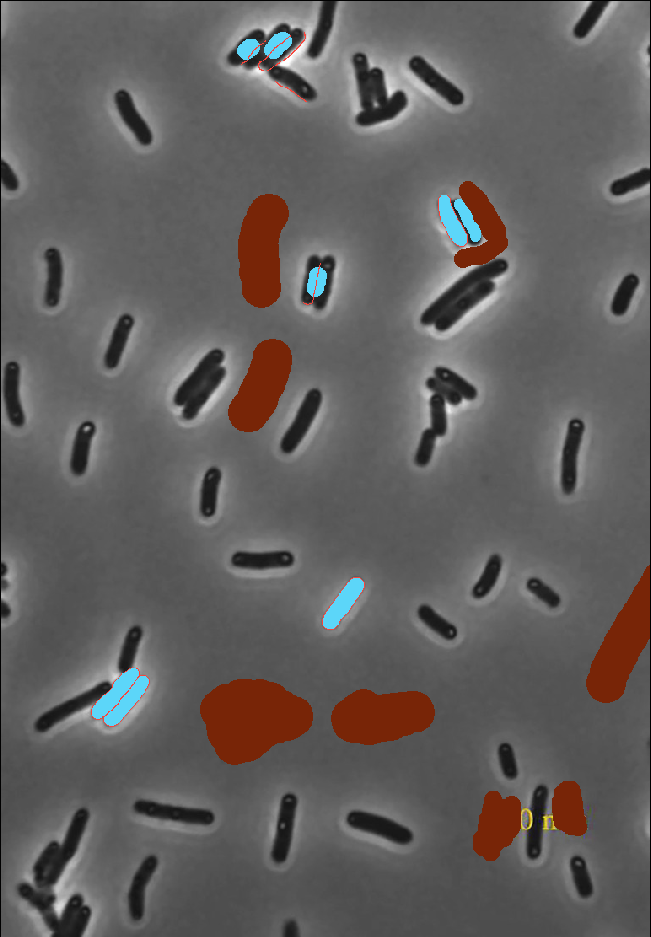}
  };
  \spy on (0.42,1.79) in node at (3.65,1.5);
\end{tikzpicture}
    \caption{Sparse labels}
    \label{fig:subfig1}
  \end{subfigure}
  \begin{subfigure}[t]{0.225\linewidth}
    \centering
    \begin{tikzpicture}[spy using outlines={rectangle, color=red, magnification=3, size=1.3cm, connect spies}]
      \node[anchor=south west,inner sep=0] (image) at (0,0) {\includegraphics[clip, trim={0pt 50 0 0}, height=\linewidth, angle=90]{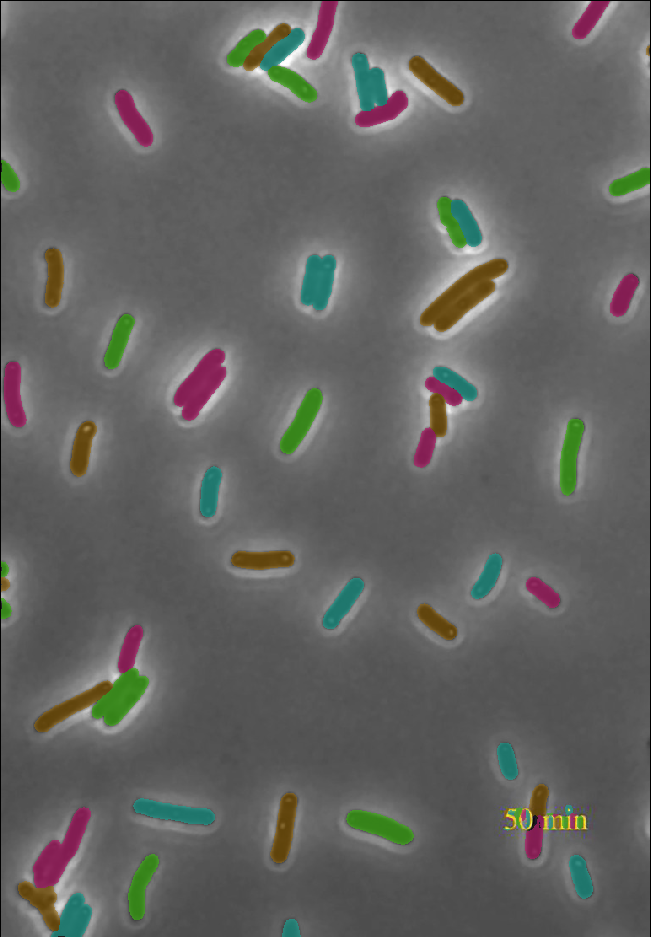}};
      \spy on (0.42,1.79) in node at (3.65,1.5);
    \end{tikzpicture}
    \caption{Omnipose bactphase}
    \label{fig:subfig2}
  \end{subfigure}
  \begin{subfigure}[t]{0.225\linewidth}
    \centering
    \begin{tikzpicture}[spy using outlines={rectangle, color=green, magnification=3, size=1.3cm, connect spies}]
      \node[anchor=south west,inner sep=0] (image) at (0,0) {\includegraphics[clip, trim={0pt 50 0 0}, height=\linewidth, angle=90]{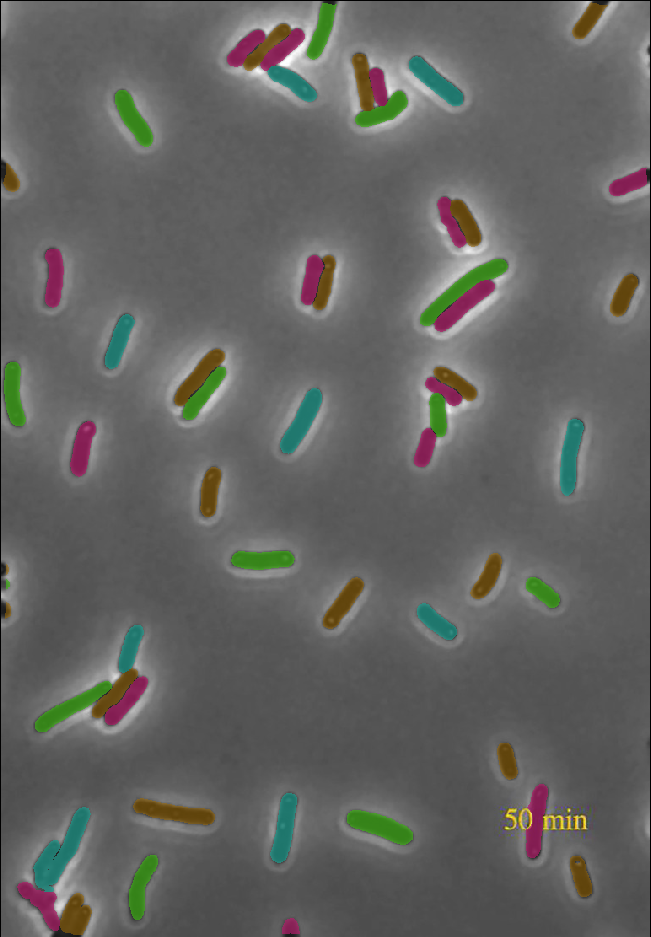}};
      \spy on (0.42,1.79) in node at (3.65,1.5);
    \end{tikzpicture}
    \caption{Sketchpose ($\approx$ 1')\label{fig:adipo_acuracy_c}}
    \label{fig:subfig3}
  \end{subfigure}
  \begin{subfigure}[t]{0.225\linewidth}
    \centering
    \includegraphics[width=1\linewidth]{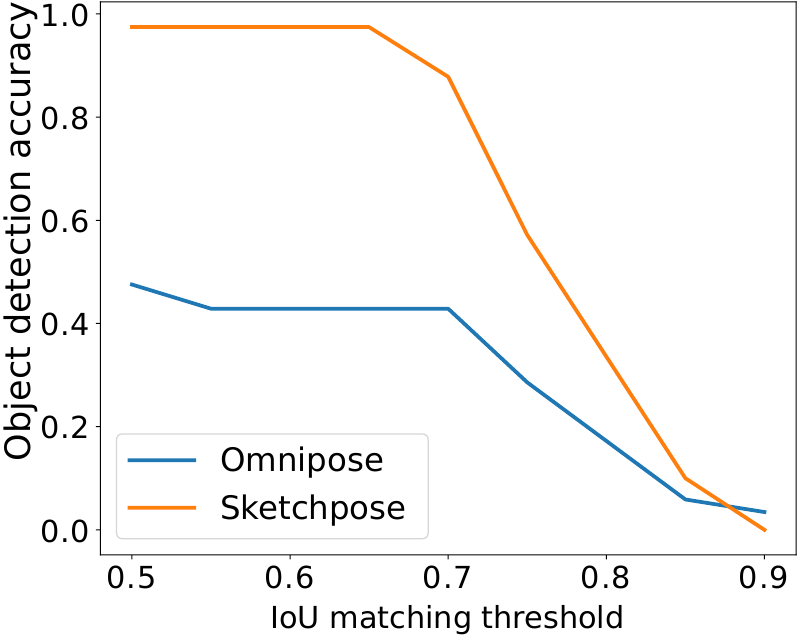}
    \caption{Quality evaluation}
    \label{fig:bact_acuracy}
  \end{subfigure}

\begin{subfigure}[c]{0.225\linewidth}
    \centering
    \begin{tikzpicture}[spy using outlines={rectangle, color=red, magnification=3, size=1.3cm, connect spies}]
      \node[anchor=south west,inner sep=0] (image) at (0,0) {\includegraphics[width=\linewidth]{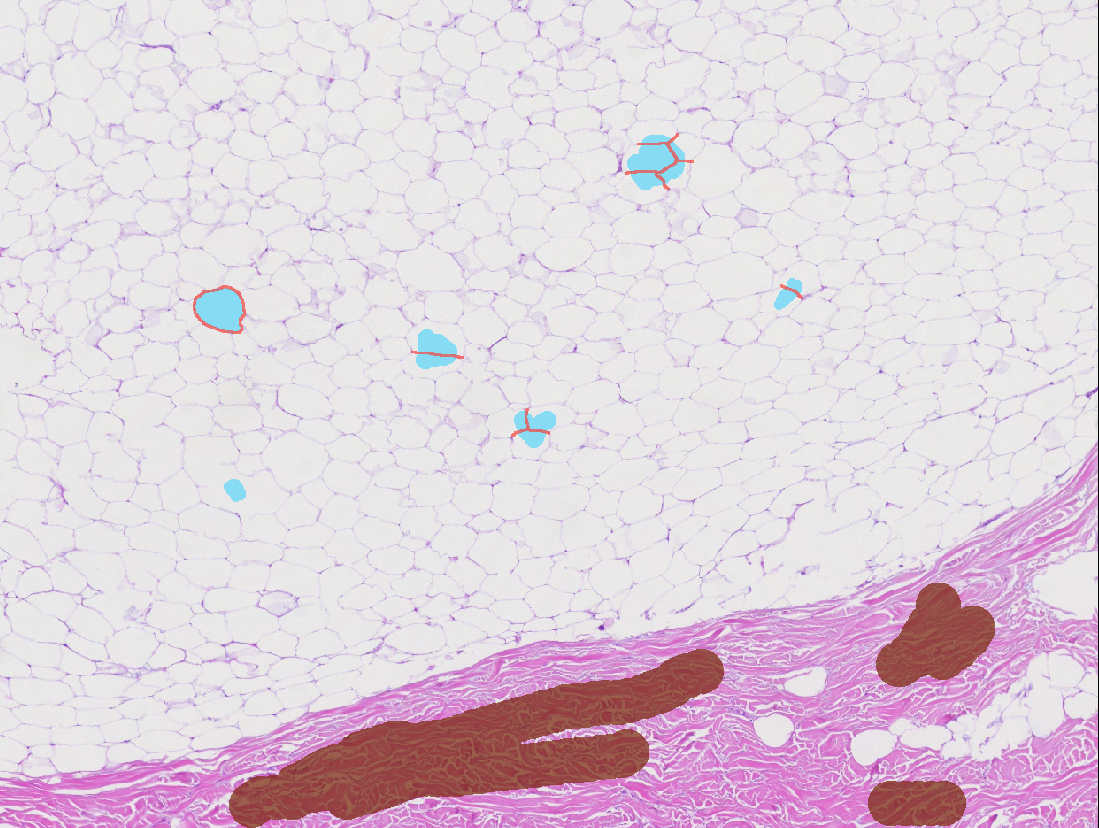}};
      \spy on (0.4,0.5) in node at (3.3, 2.3);
    \end{tikzpicture}
    \caption{Sparse labels (cropped)\label{fig:adipo_acuracy_e}}
    \label{fig:subfig2}
  \end{subfigure}
  \begin{subfigure}[c]{0.225\linewidth}
    \centering
    \begin{tikzpicture}[spy using outlines={rectangle, color=red, magnification=3, size=1.3cm, connect spies}]
      \node[anchor=south west,inner sep=0] (image) at (0,0) {\includegraphics[width=\linewidth]{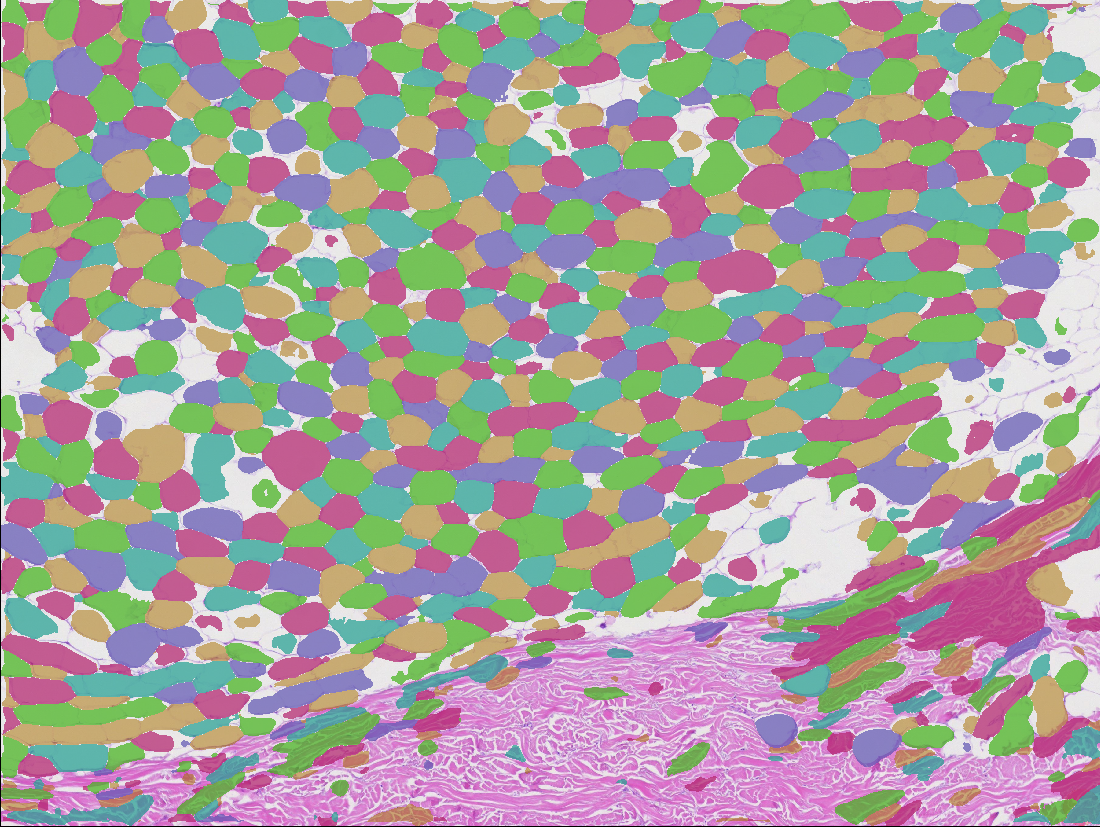}};
      \spy on (0.4,0.5) in node at (3.3, 2.3);
    \end{tikzpicture}
    \caption{Omnipose cyto2}
    \label{fig:subfig2}
  \end{subfigure}
  \begin{subfigure}[c]{0.225\linewidth}
    \centering
    \begin{tikzpicture}[spy using outlines={rectangle, color=green, magnification=3, size=1.3cm, connect spies}]
      \node[anchor=south west,inner sep=0] (image) at (0,0) {\includegraphics[width=\linewidth]{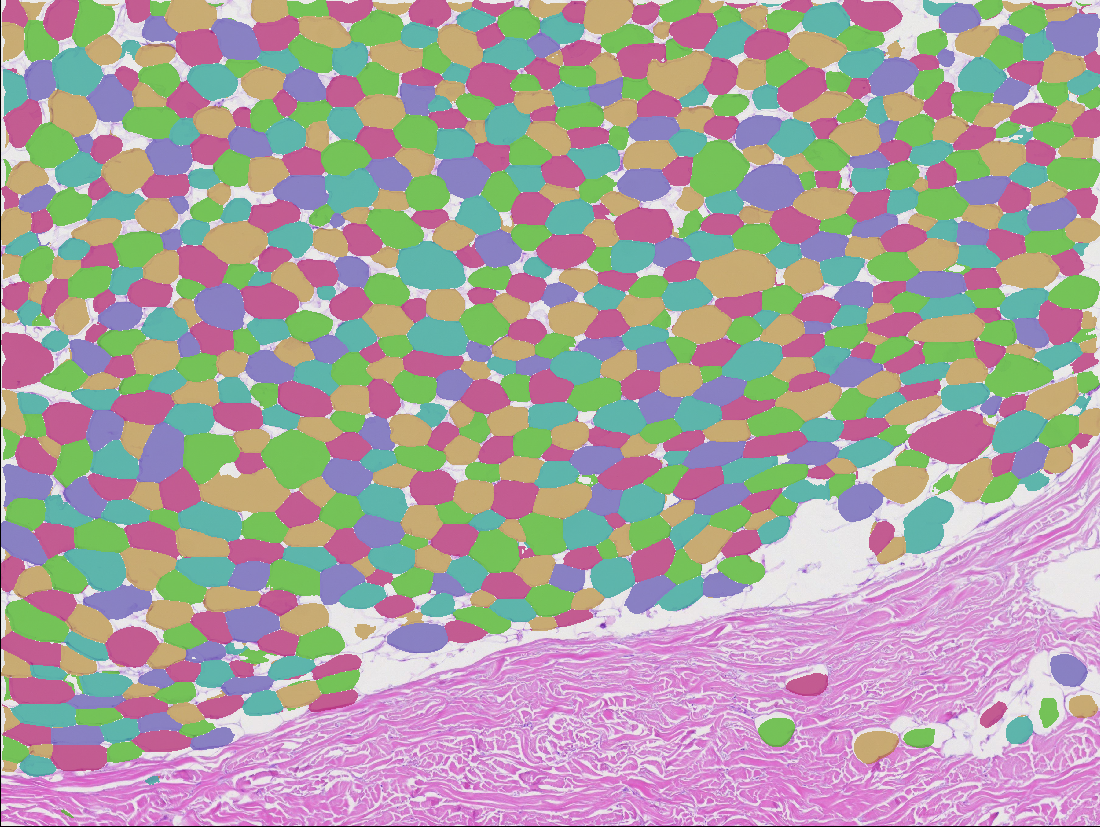}};
      \spy on (0.4,0.5) in node at (3.3, 2.3);
    \end{tikzpicture}
    \caption{Transfer learning $\approx$ 2'}
    \label{fig:subfig3}
  \end{subfigure}
  \begin{subfigure}[c]{0.225\linewidth}
    \centering
    \includegraphics[width=\linewidth]{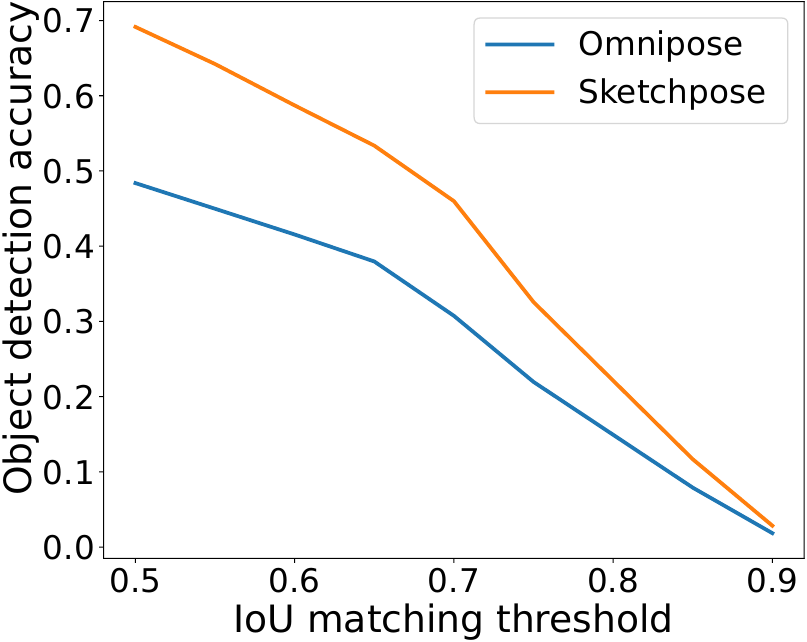}
    \caption{Quality evaluation}
  \end{subfigure}

  \begin{subfigure}[t]{0.225\linewidth}
    \centering
        \begin{tikzpicture}[spy using outlines={rectangle, color=red, magnification=3, size=1.5cm, connect spies}]
      \node[anchor=south west,inner sep=0] (image) at (0,0) {\includegraphics[width=\linewidth]{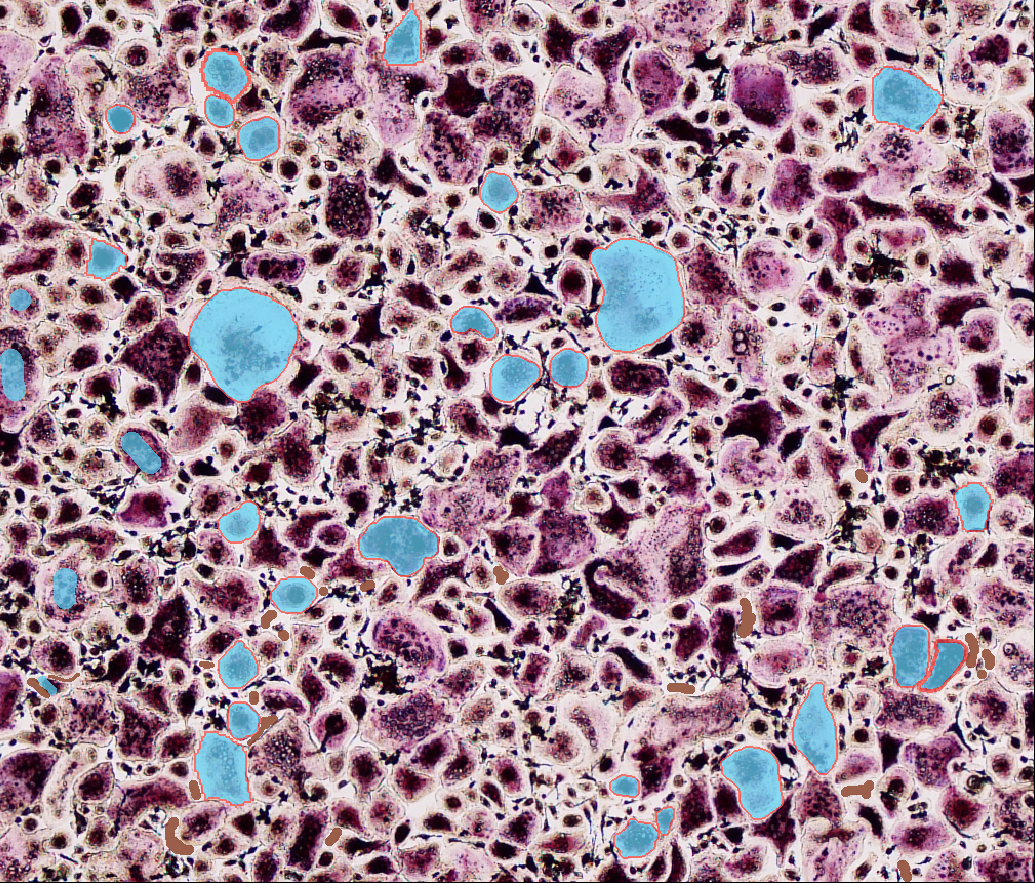}};
      \spy on (2,1) in node at (3.1,2.5);
    \end{tikzpicture}
    \caption{Sparse labels}
    \label{fig:image2}
  \end{subfigure}
    \begin{subfigure}[t]{0.225\linewidth}
    \centering
    \begin{tikzpicture}[spy using outlines={rectangle, color=red, magnification=3, size=1.5cm, connect spies}]
      \node[anchor=south west,inner sep=0] (image) at (0,0) {\includegraphics[width=\linewidth]{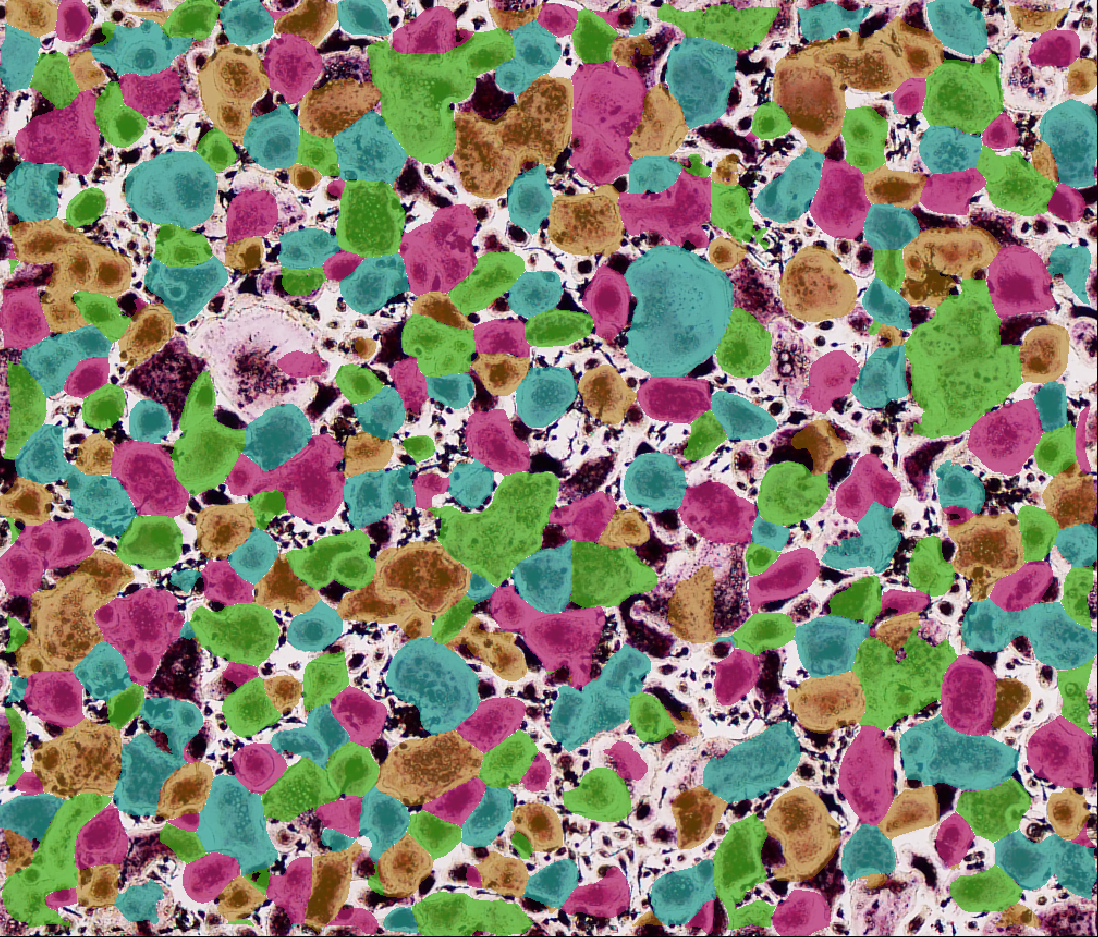}};
      \spy on (2,1) in node at (3.1,2.5);
    \end{tikzpicture}
    \caption{Omnipose cyto2}
    \label{fig:subfig2}
  \end{subfigure}
  \begin{subfigure}[t]{0.225\linewidth}
    \centering
    \begin{tikzpicture}[spy using outlines={rectangle, color=green, magnification=3, size=1.5cm, connect spies}]
      \node[anchor=south west,inner sep=0] (image) at (0,0) {\includegraphics[width=\linewidth]{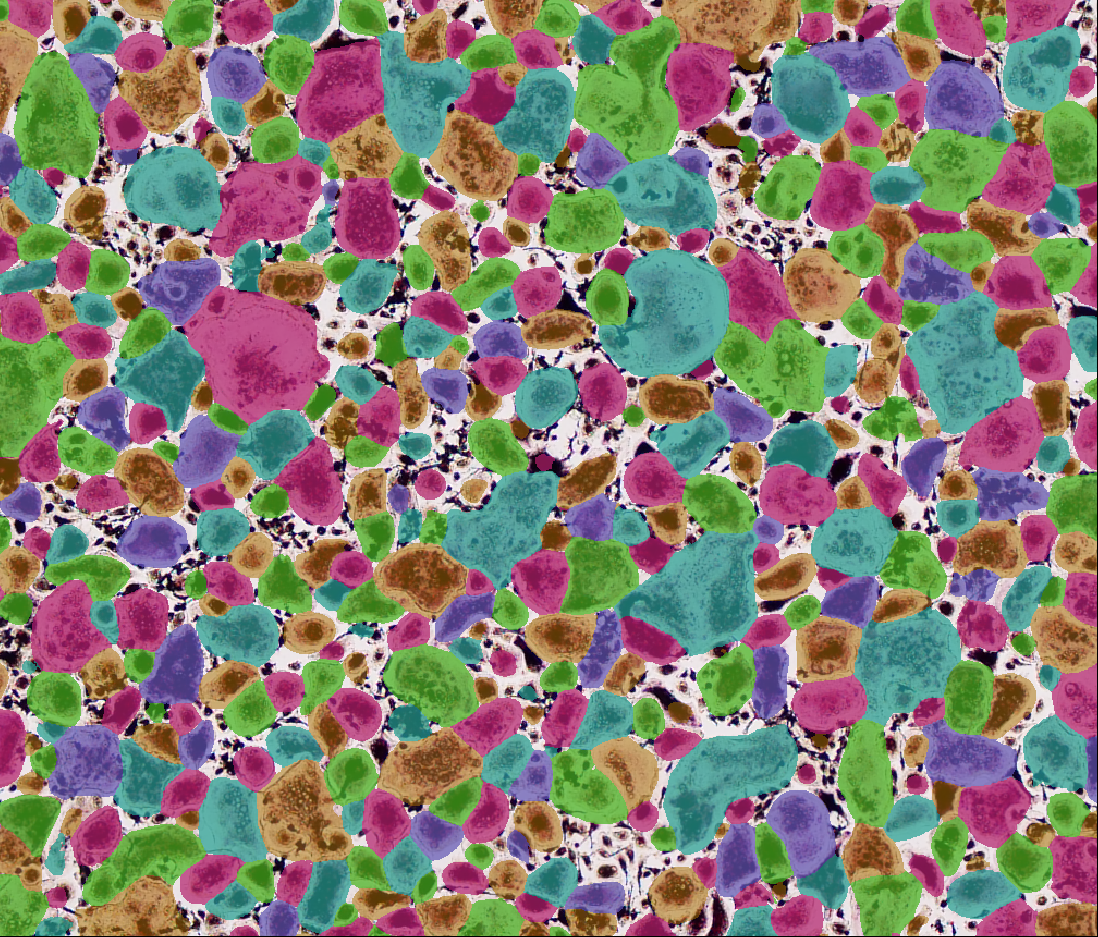}};
      \spy on (2,1) in node at (3.1,2.5);
    \end{tikzpicture}
    \caption{Transfer learning $\approx$ 5'}
    \label{fig:subfig2}
  \end{subfigure}
  \begin{subfigure}[t]{0.225\linewidth}
    \centering
  	\includegraphics[width=\linewidth]{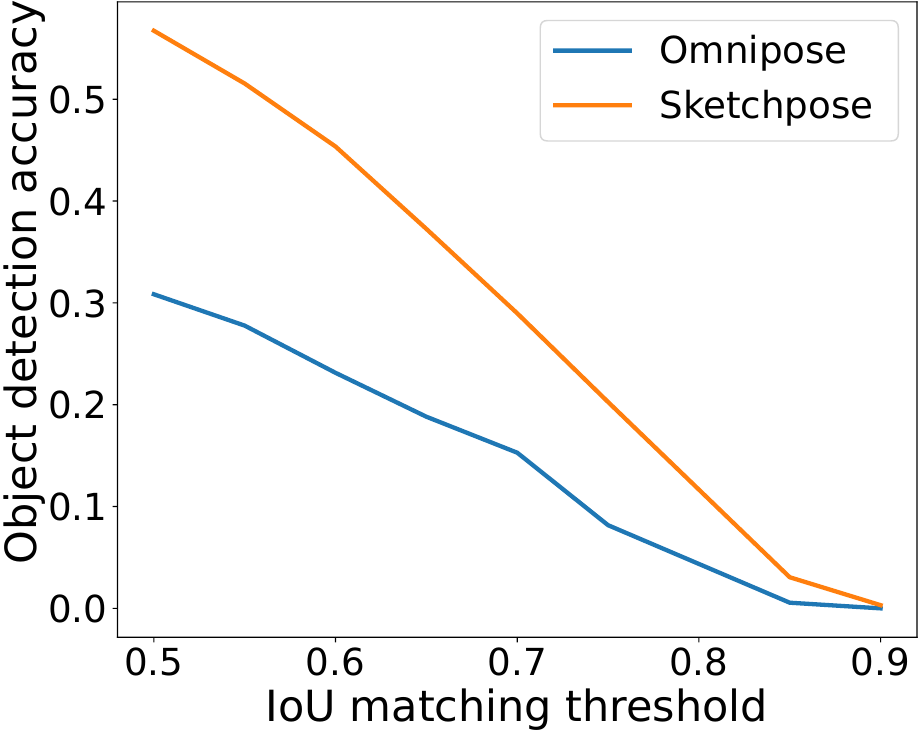}
    \caption{Quality evaluation}
    \label{fig:abs_acuracy}
  \end{subfigure}

  \caption{
Transfer learning experiment. In all cases, just a few strokes are enough to significantly improve the segmentation quality.     
\textbf{(a–d) Bacteria.} 
(a–c) Training with a few sparse labels. 
(d) Evaluation of segmentation quality. 
Omnipose: DICE = 0.81, Jaccard = 0.53. 
Sketchpose: DICE = 0.90, Jaccard = 0.72.
\label{fig:abs}
\textbf{(e–h) Adipocytes.} 
(e–g) Transfer learning from the Omnipose \texttt{bact\_phase} model. 
(h) Evaluation of segmentation quality. 
Omnipose: DICE = 0.89, Jaccard = 0.69. 
Sketchpose: DICE = 0.89, Jaccard = 0.79. 
Image adapted from \cite{ZHU2021102348}.
\label{fig:training_bact}
\textbf{(i–l) Osteoclasts.} 
(i–k) Transfer learning from the Omnipose \texttt{cyto2} model. 
(l) Evaluation of segmentation quality. 
Omnipose: DICE = 0.90, Jaccard = 0.68. 
Sketchpose: DICE = 0.95, Jaccard = 0.80.
\label{fig:adipo_acuracy}
}
\end{figure*}

\subsubsection{Adipocytes segmentation}

 The image in Figure \ref{fig:adipo_acuracy_e} shows a crop of a very large image of a skin explant provided by \href{https://www.diva-expertise.com/fr/}{DIVA Expertise}. 
 One can see a part of the dermis (in pink) and above it, adipose tissue (large white circular cells). Adipose tissue is the third skin layer after the epidermis and dermis, also known as the hypodermis. 
 Hypodermal cells (adipocytes) secrete specific molecules (e.g. adiponectin, leptin) which have a direct impact on the biology of fibroblasts present in the dermis, and also on keratinocytes present in the epidermis. 
 They are the subject of numerous studies (see \cite{bourdens2019short} and \cite{sadick2015facial} for instance). For most of the studies where skin explants are imaged, we first need to count the adipocytes number in the image, and remove any potential outliers detected in the dermis and epidermis.

 While Omnipose cyto2 results in some undersegmentation for this task, the adapted weights provided by Sketchpose yields significantly enhanced results. Annotating 6 cells and a training for 100 epochs (less than 1 minute) were sufficient to significantly improve the quality of the segmentation and to remove the outliers from the dermis (see Figure \ref{fig:adipo_acuracy_c}). Figure \ref{fig:adipo_acuracy} shows a quantitative comparison between Omnipose and Sketchpose on this example.

\subsubsection{Osteoclasts segmentation}

Osteoclasts are responsible for bone resorption, and are widely studied (see \cite{labour2016tgfbeta1} for instance) as being responsible for certain pathologies such as osteoporosis when dysfunctional. Their differentiation goes through several stages, culminating in the activated osteoclast. The latter is generally large and contains numerous nuclei. \href{http://www.atlantic-bone-screen.com/language/en/}{Atlantic Bone Screen} (ABS) company is investigating the effect of different drugs in inducing either proliferation or cell death in these activated osteoclasts, in order to regulate their population. To do so, they extract osteoclasts from biopsies, culture them, apply the drugs and image them under a bright-field microscope.

The studied image is a crop of an image containing around 20,000 cells. We can see touching cells presenting a great variety in size, shape color. 
The image is complex to segment and poses a real challenge. What is more, ABS does not want to count pre-osteoclasts (small black nuclei), but only the mature cells (according to specific nuclei criteria). Each study comprises around sixty images, hence manual counting task performed at ABS is costly and laborious.

In Figure \ref{fig:abs}, we present a qualitative depiction that underscores the enhancement in segmentation accuracy attained through transfer learning with just a few labels. 
Labeling required approximately 2 minutes, while the training process took about 5 minutes. A quantitative comparison is available in Figure \ref{fig:abs_acuracy}.

\subsection{Training from scratch on large datasets}

The aim of this experiment is to highlight the possibility to train our model on large datasets with sparse annotations.
We use two different datasets as illustrated in Figure \ref{fig:dataset_examples}.

\begin{figure*}[ht]
  \centering
  \begin{subfigure}[t]{0.19\linewidth}
    \centering
    \includegraphics[width=\linewidth]{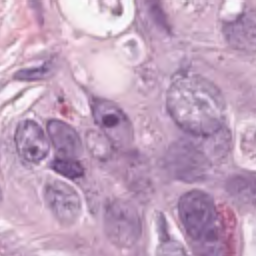}
    \label{fig:lmnet100}
  \end{subfigure}
  \hfill
  \begin{subfigure}[t]{0.19\linewidth}
    \centering
    \includegraphics[width=\linewidth]{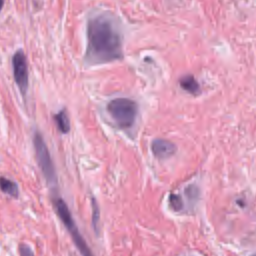}
    \label{fig:lmnet50}
  \end{subfigure}
  \hfill
  \begin{subfigure}[t]{0.19\linewidth}
    \centering
    \includegraphics[width=\linewidth]{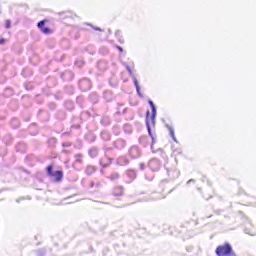}
    \label{fig:lmnet25}
  \end{subfigure}
  \hfill
  \begin{subfigure}[t]{0.19\linewidth}
    \centering
    \includegraphics[width=\linewidth]{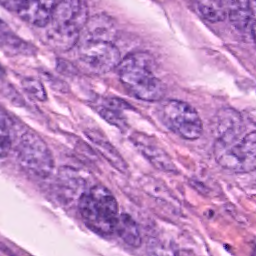}
    \label{fig:lmnet25}
  \end{subfigure}
  \hfill
  \begin{subfigure}[t]{0.19\linewidth}
    \centering
    \includegraphics[width=\linewidth]{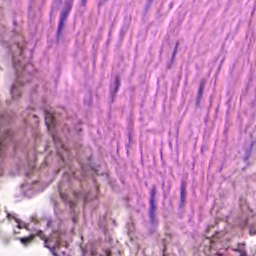}
    \label{fig:lmnet25}
  \end{subfigure}
  \hfill
  \begin{subfigure}[t]{0.19\linewidth}
    \centering
    \includegraphics[width=\linewidth]{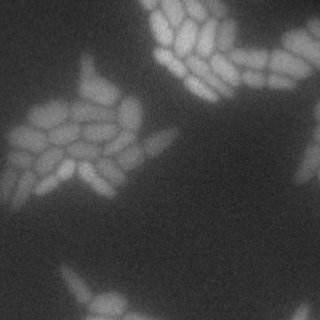}
    \label{fig:lmnet10}
  \end{subfigure}
  \hfill
    \begin{subfigure}[t]{0.19\linewidth}
    \centering
    \includegraphics[width=\linewidth]{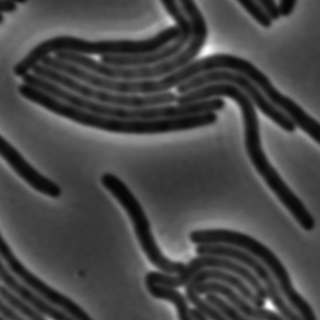}
    \label{fig:lmnet10}
  \end{subfigure}
  \hfill
    \begin{subfigure}[t]{0.19\linewidth}
    \centering
    \includegraphics[width=\linewidth]{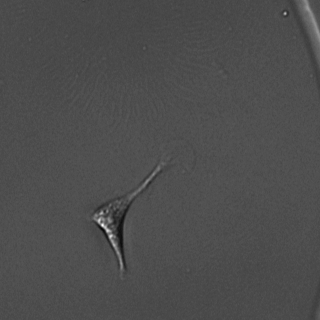}
    \label{fig:lmnet10}
  \end{subfigure}
  \hfill
  \begin{subfigure}[t]{0.19\linewidth}
    \centering
    \includegraphics[width=\linewidth]{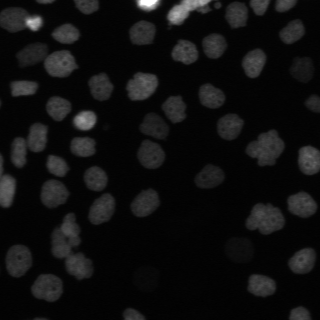}
    \label{fig:lmnet10}
  \end{subfigure}
  \hfill
  \begin{subfigure}[t]{0.19\linewidth}
    \centering
    \includegraphics[width=\linewidth]{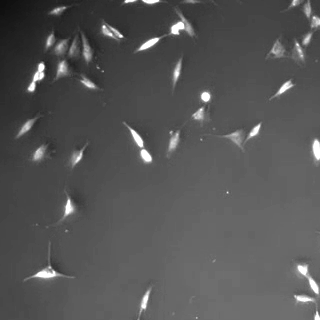}
    \label{fig:lmnet10}
  \end{subfigure}
  \caption{Top: examples of PanNuke images. Bottom: examples of MicrobeSeg images.}
  \label{fig:dataset_examples}
\end{figure*}



\paragraph{Microbeseg dataset}

The \textit{Microbeseg} dataset~\cite{microbeseg2022} contains 826 fluorescent microscopy images  of bacteria with about 30{,}000 manually annotated objects). 
It contains a mix of datasets from Omnipose and the Cell Tracking Challenge. It is publicly available \href{https://zenodo.org/records/7221152}{here}.

\paragraph{PanNuke dataset}

The PanNuke dataset~\cite{gamper2019pannuke} contains 7{,}904 image tiles of histopathology slides stained with H\&E across 19 tissue types, each with nuclear instance segmentations and five-class nuclear type annotations. It is publicly available from \href{https://www.kaggle.com/datasets/jgamper/pannuke}{Kaggle} and is widely used for benchmarking nucleus segmentation and classification algorithms.



\subsubsection{Selecting annotation subsets}

In this section, we investigate the model robustness across various annotation levels each characterized by a different percentage of annotated pixels: 10\%, 25\%, 50\%, and 100\%. 
We generate randomly binary masks by thresholding white Gaussian noise with a Gaussian filter of variance $\sigma^2$. The resulting Gaussian process is then thresholded to keep only a given proportion of pixels.

While the model is stochastic in nature, the generated data is created once and for all, enabling its deterministic reuse across multiple training sessions. 
Figure \ref{fig:training_image} shows an image accompanied by four corresponding label masks illustrating decreasing levels of annotation sparsity.

\begin{figure*}[ht]
  \centering
  \begin{subfigure}[t]{0.19\linewidth}
    \centering
    \includegraphics[width=\linewidth]{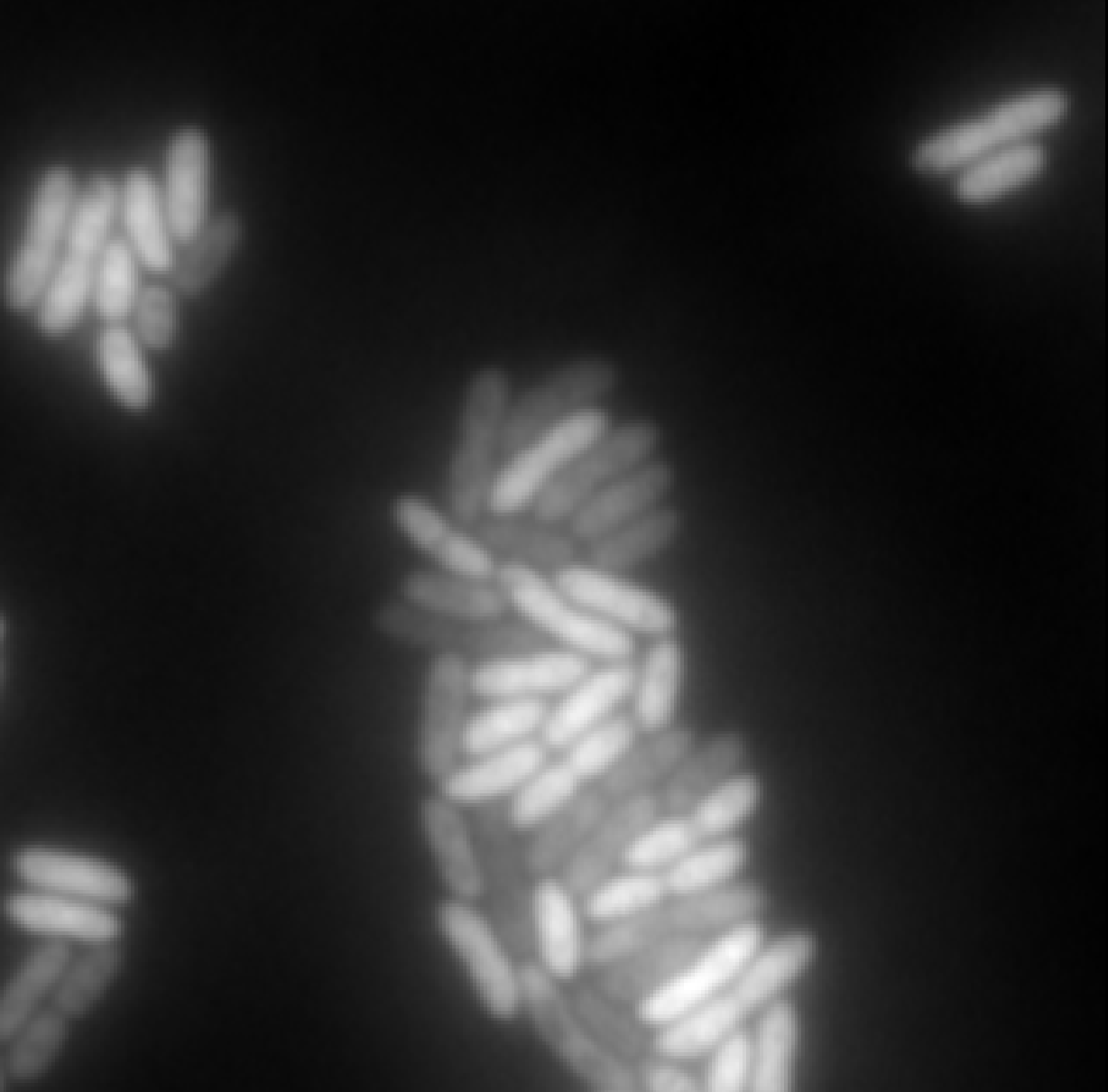}
    \caption{A raw image.}
    \label{fig:image1}
  \end{subfigure}
  \begin{subfigure}[t]{0.19\linewidth}
    \centering
    \includegraphics[width=\linewidth]{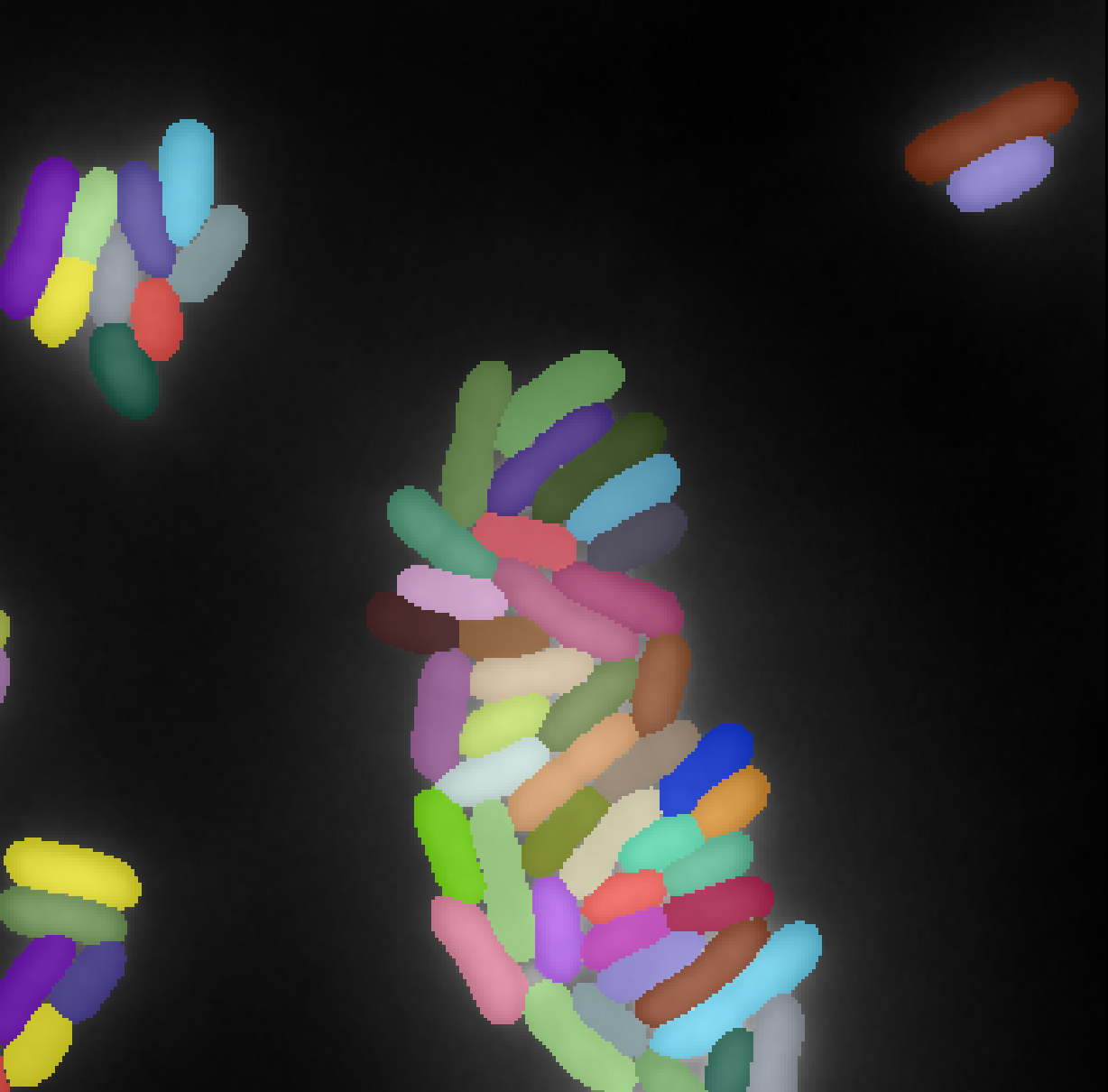}
    \caption{Fully annotated.}
    \label{fig:image2}
  \end{subfigure}
  \begin{subfigure}[t]{0.19\linewidth}
    \centering
    \includegraphics[width=\linewidth]{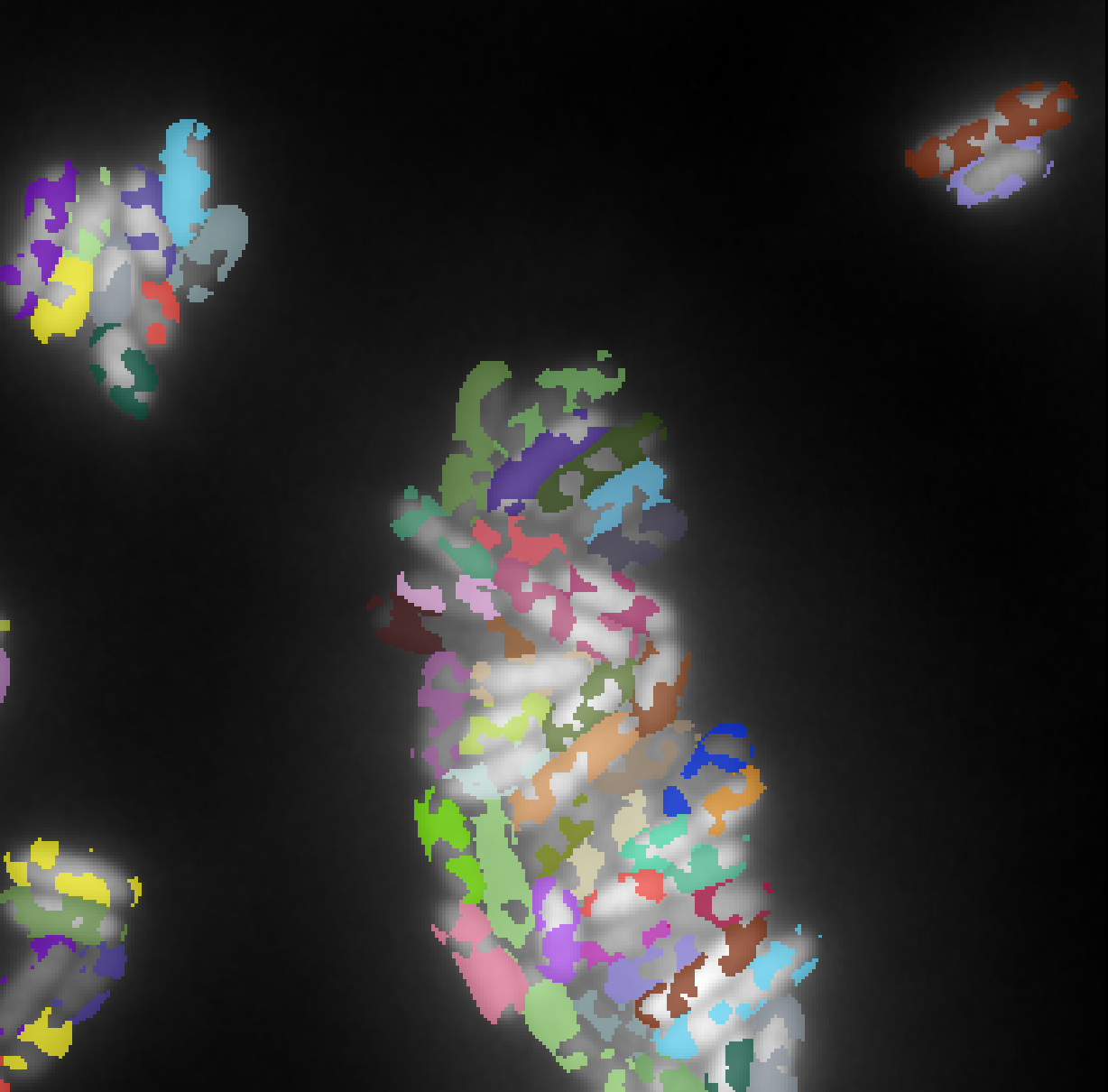}
    \caption{50\%}
    \label{fig:image3}
  \end{subfigure}
  \begin{subfigure}[t]{0.19\linewidth}
    \centering
    \includegraphics[width=\linewidth]{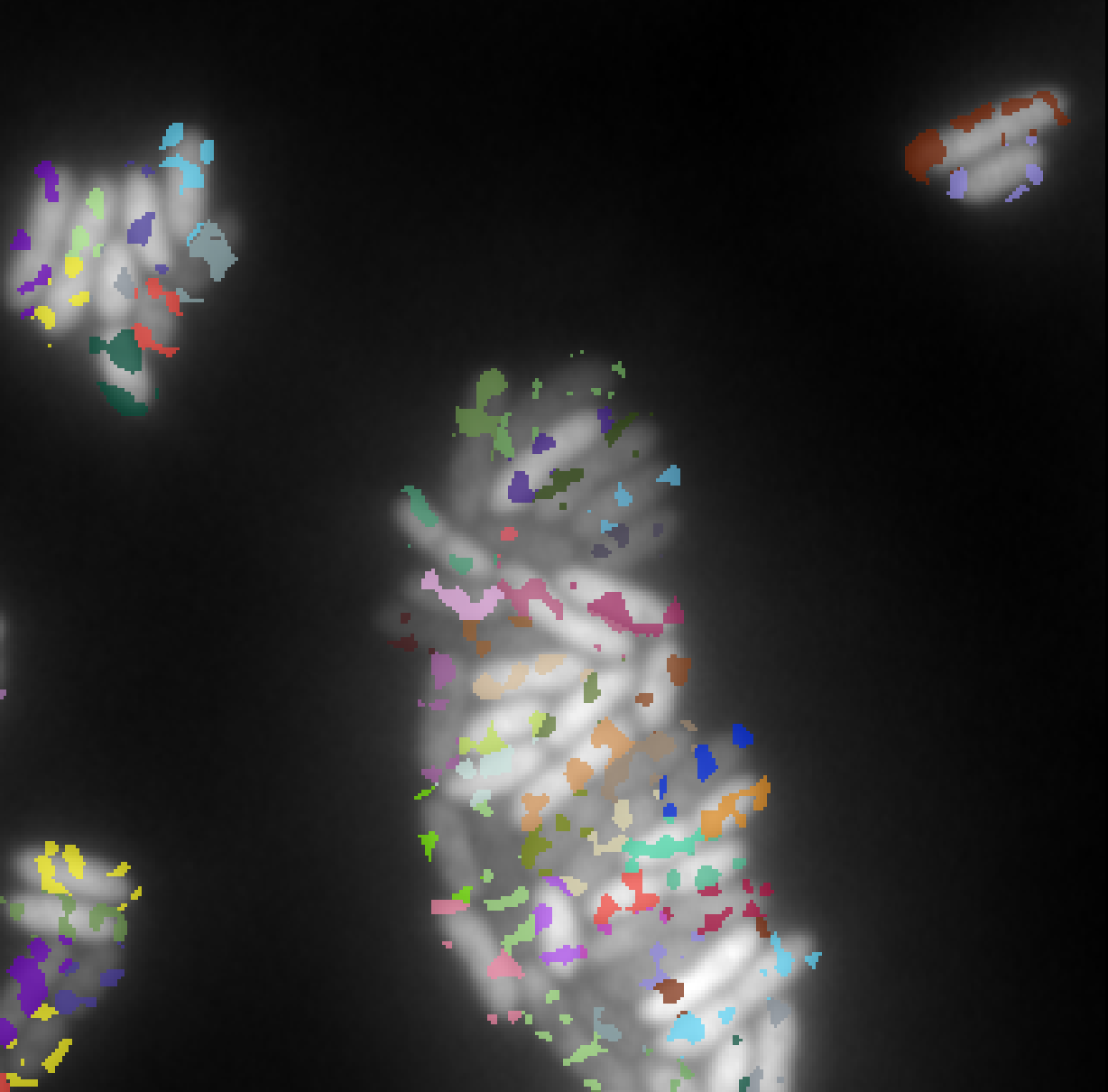}
    \caption{25\%}
    \label{fig:image3}
  \end{subfigure}
  \begin{subfigure}[t]{0.19\linewidth}
    \centering
    \includegraphics[width=\linewidth]{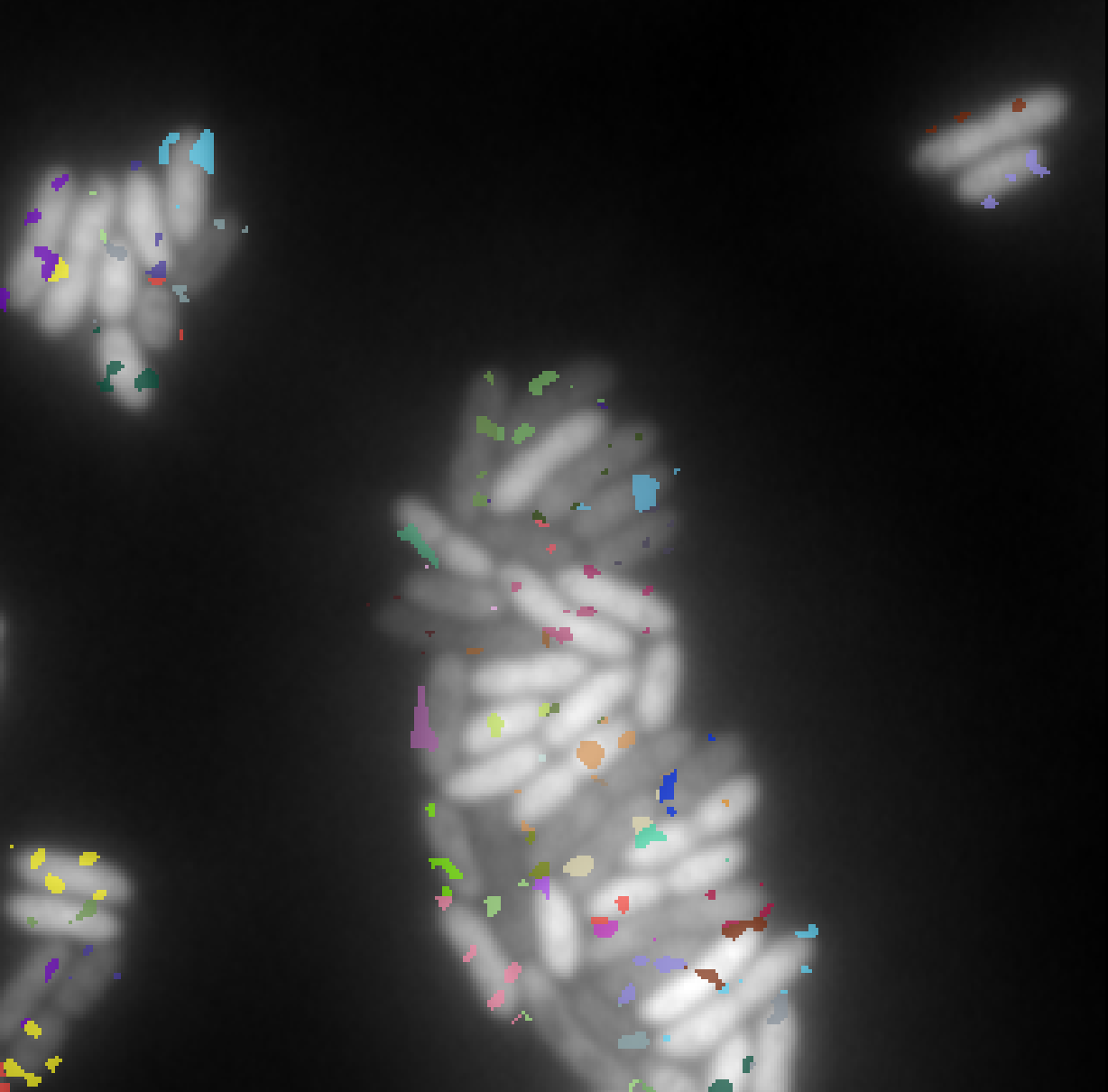}
    \caption{10\%}
    \label{fig:image3}
  \end{subfigure}
  \caption{Annotations example with four different sparsity levels and a domain regularity $\sigma=50$.}
  \label{fig:training_image}
\end{figure*}

\subsubsection{Training details}

For each dataset, we trained the Sketchpose model for 1000 epochs. 
This takes about 5 hours for the Microbeseg dataset and 30 hours for the PanNuke dataset using our Nvidia RTX5000. 
Each model was trained with the four percentages of annotated pixels we described above.
There was no data augmentation, except random cropping of the images to 224x224 pixels.
This is the size which was used in the original Omnipose model.

\subsubsection{Results}





We evaluated and compared the performance using two alternative models. 
The first one is the Cellpose 3.0 model~\cite{stringer2025cellpose3}, which regresses a distance to the objects centroids. 
The second one is the LKCell model~\cite{cui2024lkcell}, which is a better performing variant of CellVit~\cite{horst2024cellvit}, itself a variant of HoverNet~\cite{graham2019hover}.
These models are among the most popular and best performing for histopathology images such as the PanNuke dataset. 
While they perform classification and segmentation, we will just compare their ability to segment objects, as we are not interested in the classification task here.
We trained both LKCell and Cellpose 3.0 from scratch on each of the two datasets with complete annotations for 1000 epochs.

We compare the performance using other standard quality metrics used in instance segmentation. In all the metrics below, an IoU threshold of 50\% is used. 
\begin{itemize}\itemsep0pt
  \item \emph{Precision}: measures how many of the predicted positives are actually correct (i.e., the fraction of predicted segments that are true).
  \item \emph{Recall}: how many of the actual positives were correctly predicted (i.e., how complete the prediction is).
  \item \emph{F1-Score}: harmonic mean of precision and recall, balancing both.
  \item \emph{Detection Quality (DQ)}: evaluates object-level detection performance, penalizing missed or extra objects. It evaluates the ability to detect object instances correctly, regardless of segmentation quality.
  \item \emph{Segmentation Quality (SQ)}: measures how well matched objects are segmented, assuming correct pairing, reflecting the quality of the predicted segment.
\end{itemize}

The main results are reported in Table~\ref{tab:segmentation_metrics}, and several key observations emerge.
  
\paragraph{Best-performing methods}
On the MicrobeSeg dataset, Cellpose and Sketchpose (100\% annotations) deliver the best performance, while LKCell lags behind with a 15\% lower F1-score. Conversely, on PanNuke, LKCell outperforms both competitors with a 4–6\% F1-score gain, confirming its suitability for this dataset.

  \paragraph{Impact of annotation density}
For MicrobeSeg, reducing annotation density leads to a significant performance drop for Sketchpose: around 10\% at 50–25\% annotations, and up to 20\% with only 10\%. 
Depending on the application, such degradation may or may not be acceptable.

The situation is more favorable for PanNuke. Sketchpose maintains stable performance, with only a 4\% drop when reducing annotations from 100\% to 10\%. 
Given that sparse annotations likely reduce annotation time by a factor of ten, this is a promising result—especially since random sampling was used. 
In practice, targeted annotations by an expert would likely yield even better outcomes.

This contrast may stem from dataset characteristics: PanNuke contains simpler, roughly convex objects, while MicrobeSeg features elongated or irregular shapes, making annotation density more critical.

  \paragraph{IoU matching thresholds}
  To provide a refined view of the performance, we also plot the F1-score as a function of the IoU matching threshold in Figure~\ref{fig:xp_bact_metrics}.
  There, we see that the ranking between the methods is stable up to a mathcing threshold of 70\%, which is usually considered a high precision segmentation in biological imaging.
  A surprising phenomenon is that Sketchpose trained with 25\% of annotations performs better than the 50\% model on the MicrobeSeg dataset. Similarly, the 50\% model performs better than the 100\% model on the PanNuke dataset. This might indicate that carefully selected annotations can lead to better results than complete annotation, or helps reducing the influence of errors in the gold-standard database.

  Sketchpose is the first distance-based method allowing to take advantage of this observation.
 

\begin{table*}[ht]
\centering
\caption{Comparison of segmentation methods trained and tested on the MicrobeSeg and PanNuke datasets using various metrics. All values are computed with an IoU threshold of 50\%.\label{tab:segmentation_metrics}}
\begin{tabular}{lcccccc||cccccc}
\toprule
 & \multicolumn{6}{c||}{\emph{\textbf{MicrobeSeg dataset}}} & \multicolumn{6}{c}{\emph{\textbf{PanNuke dataset}}} \\
 & LKCell & Cellpose & 100\% & 50\% & 25\% & 10\% & LKCell & Cellpose & 100\% & 50\% & 25\% & 10\% \\
\midrule
Precision           & 0.73 & \textbf{0.88} & 0.77 & 0.63 & 0.68 & 0.55 & 0.82 & \textbf{0.88}  & 0.86 & 0.84 & 0.78 & 0.75 \\
Recall              & 0.65 & 0.75 & \textbf{0.89} & 0.86 & 0.83 & 0.78 & \textbf{0.84} & 0.71  & 0.71 & 0.75 & 0.77 & 0.73 \\
F1-Score            & 0.66 & 0.79 & \textbf{0.81} & 0.69 & 0.72 & 0.60 & \textbf{0.83} & 0.78  & 0.77 & 0.79 & 0.77 & 0.73\\
DICE                & 0.82 & \textbf{0.86} & 0.85 & 0.85 & 0.83 & 0.80 & \textbf{0.88} & 0.87 & 0.87 & \textbf{0.88} & 0.87 & 0.87 \\
Jaccard             & 0.73 & \textbf{0.79} & 0.75 & 0.75 & 0.72 & 0.69 & \textbf{0.81} & 0.79 & 0.79 & 0.80 & 0.79 & 0.77 \\
Det. Quality (DQ) & 0.73 & \textbf{0.88} & 0.77 & 0.63 & 0.68 & 0.55 & 0.82 & \textbf{0.88}   & 0.86 & 0.84 & 0.78 & 0.75 \\
Seg. Quality (SQ) & 0.73 & \textbf{0.79} & 0.75 & 0.75 & 0.72 & 0.69 & \textbf{0.81} & 0.79 & 0.79 & 0.80 & 0.79 & 0.77 \\
\bottomrule
\end{tabular}
\end{table*}

\begin{figure*}[ht]
  \centering
  \begin{subfigure}[t]{0.45\linewidth}
    \centering
    \includegraphics[width=\linewidth]{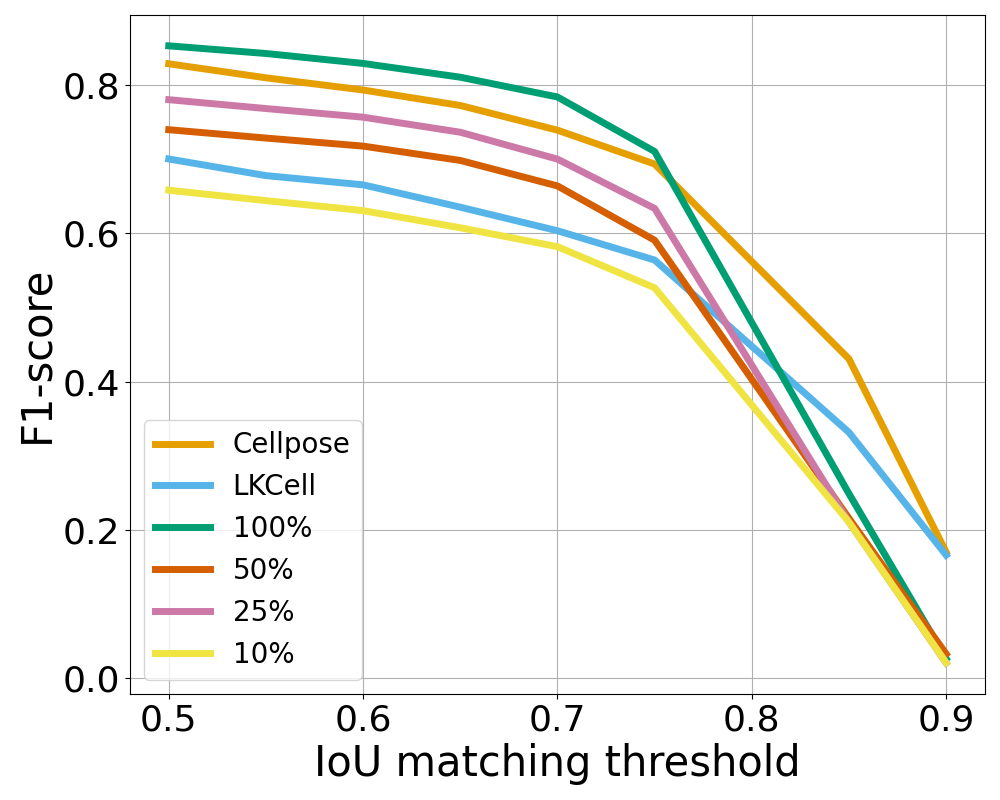}
    \caption{F1-score over the Microbeseg test dataset.}
    \label{fig:ap}
  \end{subfigure}
  \hfill
  \begin{subfigure}[t]{0.45\linewidth}
    \centering
    \includegraphics[width=\linewidth]{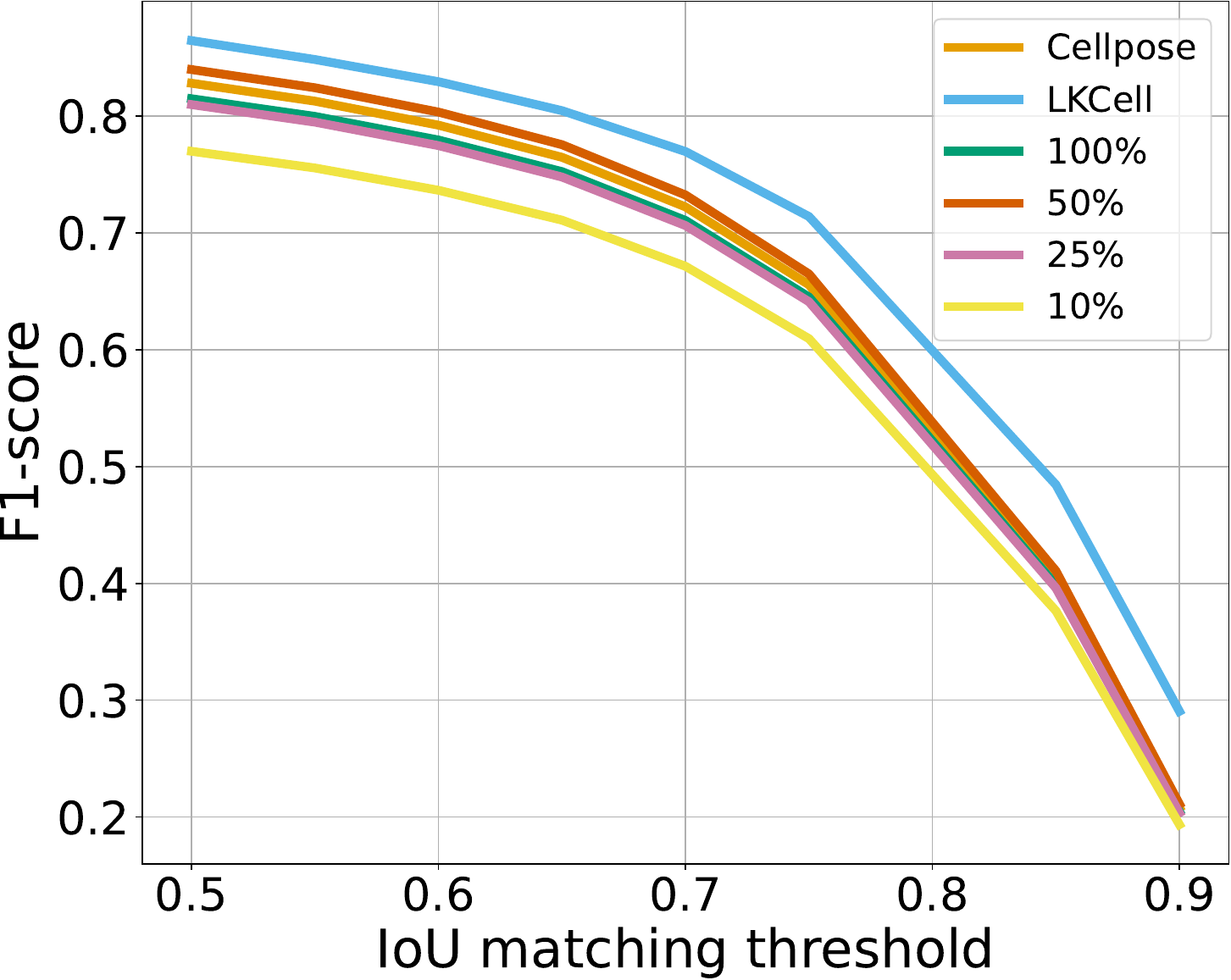}
    \caption{F1-score over the PanNuke test dataset.}
    \label{fig:ap}
  \end{subfigure}
  \caption{Evaluation of segmentation F1-score over Microbeseg's and PanNuke's test datasets as a function of the percentage of annotated pixels compared to Cellpose and LKCell.}
  \label{fig:xp_bact_metrics}
\end{figure*}

\section{Discussion \& conclusion}

We introduced Sketchpose, an open-source plugin to extend the applicability of Omnipose to partial annotations. 
From a methodological aspect, we developed a theory making it possible to use distance functions, despite having only access to partial information on the objects boundaries. 
From a more practical viewpoint, we developed an interactive interface within Napari, which facilitates efficient online learning with a real-time visualization of the training progress. The multi-threaded implementation allows users to continue annotating while the neural network trains or infers.

The new training procedure was tested in three different frames: i) training a neural network from scratch and just a few strokes, ii) improving the weights of a pre-trained network (a.k.a. transfer learning or human in the loop), iii) training with massive, but partial annotations.  

For point i), frugal annotation works surprisingly well on a few test cases despite really limited information. A dozen strokes are already enough to provide results on par -- or better -- than pre-trained networks.

For point ii), our experiments demonstrated the potential benefits of using transfer learning. That is, starting with a pre-trained Omnipose models, we can further refine it using our methodology. 

As for point iii), the conclusions are diverse. For datasets containing simple object shapes, such as PanNuke, it seems that a limited number of annotations (down to 25\%) is sufficient to achieve results on par or even better than complete annotations. For more complex objects, it seems that complete annotations are still preferable. 
These conclusions should be validated on a case by case basis, but the ability to annotate while training make it possible to take the minimal amount of annotation time for a given task. 

The method also shows a few limitations. First, it would benefit from faster training times to make the method even more interactive. We plan to improve this aspect in the forthcoming versions. Second, it is important to mention that our formalism is currently restricted to the two dimensional setting with two labels (background / foreground). 
Extending the methodology to numerous classes is rather straightforward, and the proposed ideas extend directly to this case. 
However, the proposed strategy do not extend to 3D directly. It could be used if the user was able to delineate a surface surrounding the objects of interest, but not just curves in 2D. Indeed, this would result in an empty valid distance set (see Theorem \ref{thm:valid_dist_set}) and unadapted loss functions. This limitation of the method must be put into perspective by the fact that even the Cellpose 3D model is based on 2D predictions only, which are aggregated in post-processing.

In summary, the proposed method demonstrated numerous qualities in 2D for partial annotations. We showed that it is possible to train complex networks with a few sketches, reducing the annotation burden significantly. Further developments are needed to accelerate the training process and for a multi-class extension in 3D.



\acks{C. Cazorla was a recipient of ANRT (Agence Nationale pour la Recherche et la Technologie) in the context of the CIFRE Ph.D. program (N°2020/0843) with \href{https://www.imactiv-3d.com/}{Imactiv-3D} and Institut de Mathématiques de Toulouse (IMT). 
P. Weiss acknowledges a support from \href{https://anr.fr/ProjetIA-19-P3IA-0004}{ANR-3IA Artificial and Natural Intelligence Toulouse Institute} ANR-19-PI3A-0004 and from the \href{https://anr.fr/Project-ANR-21-CE48-0008}{ANR Micro-Blind} ANR-21-CE48-0008.
This work was performed using HPC resources from GENCI-IDRIS (Grant 2021-AD011012210R1). 

We are grateful for the information provided by Kevin John Cutler about the original Omnipose implementation. 
The authors acknowledge \href{http://www.atlantic-bone-screen.com/language/en/}{Atlantic Bone Screen} for providing the osteoclasts image and \href{https://www.diva-expertise.com/fr/}{DIVA Expertise} for providing the adipocytes image.}


\ethics{The work follows appropriate ethical standards in conducting research and writing the manuscript, following all applicable laws and regulations regarding treatment of animals or human subjects.}

\coi{We declare we don't have conflicts of interest.}

\data{The code is available there: \url{https://bitbucket.org/koopa31/napari-sketchpose/src/master/}.
The documentation is available there: \url{https://sketchpose-doc.readthedocs.io/en/latest/}.}

\bibliography{refs}







\appendix

\section{Proof of the valid distance set theorem}\label{appendixA}
We start with a basic observation.
\begin{proposition}[Properties of the distance function\label{prop:distance}]\ 
\begin{itemize}
    \item $\Ac_1\subset \Ac_2 \Rightarrow \forall \xb\in \Xc, \dist(\xb,\Ac_2) \leq \dist(\xb, \Ac_1)$.
    \item $\Ac_1\subset \Ac_2$ and $\xb\in \Ac_1$ $\Rightarrow$ $\dist(\xb,\partial \Ac_1)\leq \dist(\xb,\partial \Ac_2)$.
\end{itemize}
\end{proposition}
\begin{proof}
The first item is direct:
    \begin{align*}
         \dist(\xb,\Ac_1) &= \inf_{\xb'\in \Ac_1}  \dist(\xb',\xb) \\
         &\geq \inf_{\xb'\in \Ac_2} \dist(\xb',\xb) = \dist(\xb,\Ac_2).
    \end{align*}

Here is one proof of the second iten by separating the two cases: either $x \in \mathring{\Ac_1}$ or $x \in \partial \Ac_1$. 
\begin{itemize}
\item \emph{Case 1: $x \in \partial \Ac_1$.} This case is trivial since $\dist(\xb,\partial \Ac_1) = 0 \leq \dist(\xb,\partial \Ac_2)$ by positivity of the distance.

\item \emph{Case 2: $x \in \mathring{\Ac_1}$.} In that case, the key argument is to show that the open ball of radius $\dist(\xb,\partial \Ac_1)$ centered in $x$ is included in $\mathring{\Ac_1}$. Precisely
\begin{equation}
\label{eq:ball_in_interior}
\tilde{\Bc} \eqdef \Bc\left(x, \dist(\xb,\partial \Ac_1)\right) \subseteq \mathring{\Ac_1} \subseteq \mathring{\Ac_2}.
\end{equation}
Indeed having Equation \eqref{eq:ball_in_interior} established implies by contraposition that
\begin{equation}
\partial \Ac_2 \subseteq \mathring{\Ac_2}^c \subseteq \tilde{\Bc}^c
\end{equation}
where the first inclusion is given by $\partial \Ac_2 \eqdef \bar{\Ac_2} \setminus \mathring{\Ac_2} \subseteq \mathring{\Ac_2}^c$.
Therefore, taking infimum with respect to these sets, implies the following inequalities and by the way the intended result.
\begin{align*}
\dist(\xb,\partial \Ac_2) &= \inf_{z \in \partial \Ac_2} \| z - x\| \\
	&\geq \inf_{z \in \mathring{\Ac_2}^c} \| z - x\|  \\
	&\geq \inf_{z \in \tilde{\Bc}^c} \| z - x\|  = \dist(\xb,\partial \Ac_1).
\end{align*}

So let's prove Equation \eqref{eq:ball_in_interior}: by contradiction, assume that there exists $z \in \tilde{\Bc} \cap \mathring{\Ac_1}^c$. 
Notice that $[x,z] \cap \mathring{\Ac_1}^c$ is a compact set (here $[x,z]$ denotes the closed segment between the points $x$ and $z$). 
Thus 
\begin{equation*}
z^* \eqdef \argmin_{x' \in [x,z] \cap \mathring{\Ac_1}^c} \| x - x'\|
\end{equation*} 
is well defined and the semi-open segment $[x, z^*[$ is included in $\mathring{\Ac_1}$. This implies that $z^* \in \partial \Ac_1$ since the sequence $z_n \eqdef x + (1-\frac{1}{n}) \left( z^* - x\right) \in [x, z^*[ \subseteq \mathring{\Ac_1}$ converges to $z^*$. The contradiction comes from 
\begin{equation*}
\dist(\xb,\partial \Ac_1) \leq \| z^* - x\| \leq \| z - x \| < \dist(\xb,\partial \Ac_1).
\end{equation*}
The first inequality holds because $z^* \in \partial \Ac_1$, the second one because $z^* \in [x,z]$ and last one because $z \in \tilde{\Bc}$.

By proof by contradiction, Equation \eqref{eq:ball_in_interior} holds. 
\end{itemize}
In conclusion, in all cases the inequality is verified.
\end{proof}

Theorem \ref{thm:valid_dist_set} can be proven in two steps. 
    First, notice that the inclusion $\Bc \subseteq \Ec$ (Assumption \ref{ass:ass1}) and the first bullet in Proposition \ref{prop:distance} implies that $\dist(\xb, \Ec) \leq \dist(\xb,\Bc)$ for any $\xb\in \Xc$.

    Let's establish the converse inequality. Let $\xb$ denote an arbitrary point in $\Dc$. Aiming for a proof by contradiction, assume that $\dist(\xb, \Ec) < \dist(\xb, \Bc)$.
     We can proceed by separating two cases:
    \begin{itemize}
        \item \emph{Case 1: $\dist(\xb, \Ec) = 0$}. This implies that $\xb \in \Ec$ since the set $\Ec$ is closed as a finite union of closed sets $\partial \Xc_{i,j}$. Moreover, as $\xb$ belongs to $\Dc$, in particular $\xb$ belongs to $\Sc$. It is sufficient to apply \eqref{eq:boundaries_drawn} and obtain $\xb \in \Ec \cap \Sc \subseteq \Bc$ which is inconsistent with $\dist(\xb, \Bc) > 0$. 
        \item \emph{Case 2: $r \eqdef \dist(\xb, \Ec) > 0$}. The point $\xb$ verifies $\dist(\xb, \Bc) \leq \dist(\xb, \Cc \Bc)$ as $\xb \in \Dc$. Let us define $r \eqdef \dist(\xb,\Ec) > 0$ and 
		\begin{align*}
		\varepsilon &\eqdef \dist(\xb,\Cc \Bc) - \dist(\xb,\Ec) \\
		 &\geq \dist(\xb,\Bc) - \dist(\xb,\Ec) > 0
		\end{align*}		        
        by assumption. Since $\xb$ belongs to $\Sc$, there exists $i_0 \in \{0,1\}$ such that $\xb \in \Sc_{i_0}$. Because $\dist(\xb,\Ec) = r$, there exists a point $\zb \in \Ec$ such that $r \leq \|\xb - \zb \|_2 = r + \varepsilon/2$
    
    	\begin{itemize}
    		\item \emph{Case 2.a: $\zb \in \Sc_{i_0}$.} By assumption \eqref{eq:boundaries_drawn}, the contradiction comes quickly since now
    		\begin{equation}
    		\zb \in \Sc_{i_0} \cap \Ec \subseteq \Sc \cap \Ec \subseteq \Bc \subseteq \Cc \Bc
    		\end{equation} 
    		and this implies the contradictive inequality 
    		\begin{equation*}
    		r + \varepsilon = \dist(\xb, \Cc \Bc) \leq \|\xb-\zb\|_2 \leq r + \varepsilon/2.
\end{equation*}
    		\item \emph{Case 2.b: $\zb \notin \Sc_{i_0}$.}  In that case, we may define the point $\yb$ on the line $[\xb,\zb]$ which is the nearest from the point $\xb$ and also in $\partial \Sc_{i_0}$. Since $\yb \in \partial \Sc_{i_0} \subseteq \Cc\Bc$, it implies a contradiction as intended:
    		\begin{align*}
    		r+\varepsilon &= d(\xb, \Cc\Bc) \leq \|\xb, P(s)\|_2 \\
    		&= s \| \xb-\zb\|_2 \leq \| \xb-\zb\|_2 = r+ \varepsilon/2.
    		\end{align*}
    		The point $\yb$ is defined as $P(s)$ where the map $P : t \mapsto t\zb + (1-t) \xb$ assigns to each scalar $t \in [0,1]$ a point $P(t)$ of the line and set
    		\begin{equation}
    		s \eqdef \inf_{P(t) \in \Sc_{i_0}} t.
			\end{equation}
			Since $P(0) = \xb \in \Sc_{i_0}$ and $P(1) = \zb \notin \Sc_{i_0}$, the scalar $s$ is well defined. The remaining task is to show that $P(s) \in \partial \Sc_{i_0}$. The argument works by construction and with a topological argument. Indeed, by definition of the infimum, there exists a sequence $0 < \eta_n \rightarrow 0$ such that $P(s + \eta_n) \notin \Sc_{i_0}$, thus $P(s) \notin \mathring{\Sc_{i_0}}$. Also by definition, for all $0 \leq \eta < s$, $P(\eta) \in \Sc_{i_0}$, thus $P(s) \in \bar{\Sc_{i_0}}$
    	\end{itemize}
    \end{itemize}
In both cases, the assumption $\dist(\xb, \Ec) < \dist(\xb, \Bc)$ leads to a contradiction. We deduce that $\dist(\xb, \Ec) \geq \dist(\xb, \Bc)$.

The second inequality in Theorem \ref{thm:valid_dist_set} is a consequence of the property \eqref{eq:eq_ass1}. Indeed, this property implies that we can separate the strokes $\Sc_i$ into connected components $\Sc_{i,j}$. These are subsets of the connected components $\Xc_{i,j'}$ for some $j'$ depending on $j$.
The inequality is then just a consequence of Proposition \ref{prop:distance}. 
%
%
%

\end{document}